%% file: neurips_2026.tex
\documentclass{article}

\usepackage[preprint]{neurips_2026}

\usepackage[utf8]{inputenc} 
\usepackage[T1]{fontenc}    
\usepackage{hyperref}       
\usepackage{url}            
\usepackage{booktabs}       
\usepackage{amsfonts}       
\usepackage{nicefrac}       
\usepackage{microtype}      
\usepackage{xcolor}         

\usepackage{multirow}
\usepackage{algorithm}
\usepackage{algorithmic}
\usepackage{amsmath,amssymb,amsthm}
\usepackage{wrapfig}
\usepackage{caption}
\usepackage{tikz}
\usetikzlibrary{arrows.meta, positioning}

\newtheorem{proposition}{Proposition}

\usepackage{makecell}
\usepackage{graphicx}
\usepackage{rotating}
\usepackage{adjustbox}
\usepackage[table]{xcolor}
\definecolor{secondbest}{RGB}{0,0,255}
\definecolor{darkgreen}{RGB}{0, 100, 0}
\newcommand{\best}[1]{\textbf{#1}}

\title{PnP-CF: A Consensus Flow-Matching Solver for General Inverse Problems}

\author{%
  Dibyanshu Kumar \\
  TH Würzburg-Schweinfurt, CAIRO\\
  University of W\"{u}rzburg \\
  \texttt{dibyanshu.kumar@thws.de} \\
  \And
  Magda Gregorov\'{a} \\
  TH Würzburg-Schweinfurt, CAIRO\\
  \texttt{magda.gregorova@thws.de} \\
}

\input{custom_commands}

\begin{document}

\maketitle

\begin{abstract}
Plug-and-Play (PnP) methods using continuous-time deep generative models as learned data distribution priors enable synthesizing high-quality solutions of model-based inverse problems.
The quality of the solutions depends critically on the interplay and consensus between the inverse problem solver and the generative model sampler.
We revisit a recently introduced flow-based PnP framework, which we refer to as PnP-CF (\emph{Consensus Flow}), and extend its application to nonlinear inverse problems, including those with learned neural-network forward operators.
PnP-CF combines a flow-matching denoiser with (i) ADMM-like dual variable updates and (ii) a time-adaptive pullback onto the intermediate flow probability path that combines a stochastic and a deterministic component.
We discuss how these two features stabilize the algorithm and promote consensus, and demonstrate on a set of experiments that PnP-CF not only solves standard linear inverse problems with quality at least comparable to state-of-the-art baselines, but also extends gracefully to nonlinear inverse problems, including forward operators realized by pretrained neural networks, for which existing methods produce unsatisfactory results or fail completely.

\end{abstract}

\input{introduction.tex}

\input{background.tex}

\input{methods.tex}

\input{experiments.tex}

\section{Conclusion and Limitations}
\label{conclusion}


We extended the flow-based PnP framework of \citet{kumarpnpadmm}, here termed PnP-CF, to nonlinear inverse problems, including forward operators realized by pretrained neural networks.
Its two key mechanisms, the ADMM-like dual variable and the time-adaptive pullback onto the flow probability path, promote consensus between the forward and prior-consistency updates, and our experiments show that this consensus carries over to nonlinear forward-consistency problems.
Across linear and nonlinear tasks, PnP-CF achieves strong reconstruction quality and measurement consistency while retaining the low runtime and memory footprint of PnP-style methods; on nonlinear problems in particular, none of the evaluated baselines matches its performance consistently across tasks.
Segmentation-guided reconstruction further shows that PnP-CF can handle highly non-injective learned forward operators beyond standard analytical measurement models.
A limitation of the current work is the lack of rigorous convergence guarantees.
Although the primal residuals provide strong empirical evidence of convergence, formal analysis remains challenging because the learned flow denoiser is implicit, nonlinear, and nonconvex.
The current formulation also assumes differentiable forward operators, since forward consistency relies on gradient based optimization.
Extending PnP-CF to nondifferentiable forward models is an important direction for future work.

\bibliographystyle{plainnat}
\bibliography{references}

\clearpage
\appendix

\input{appendix_main.tex}



\end{document}

%% file: custom_commands.tex
\usepackage{amsmath,amsfonts,bm}

\def\eqref#1{equation~\ref{#1}}

\def\1{\bm{1}}

\def\rvepsilon{{\mathbf{\epsilon}}}

\def\rvb{{\mathbf{b}}}

\def\rvu{{\mathbf{i}}}

\def\rvr{{\mathbf{r}}}

\def\rvu{{\mathbf{u}}}
\def\rvv{{\mathbf{v}}}

\def\rvx{{\mathbf{x}}}
\def\rvy{{\mathbf{y}}}
\def\rvz{{\mathbf{z}}}

\def\vtheta{{\bm{\theta}}}

\def\mI{{\bm{I}}}

\DeclareMathAlphabet{\mathsfit}{\encodingdefault}{\sfdefault}{m}{sl}
\SetMathAlphabet{\mathsfit}{bold}{\encodingdefault}{\sfdefault}{bx}{n}



%% file: introduction.tex
\section{Introduction}
\label{intro}

A fundamental task in model-based inverse problems is recovering an unknown signal
$\rvx \in \mathbb{R}^m$ from measurements $\rvy \in \mathbb{R}^n$
satisfying a (noisy) forward model
\begin{equation}
\rvy = \mathcal{A}(\rvx,\boldsymbol{\eta}),
\end{equation}
where $\mathcal{A}$ denotes the measurement process and $\boldsymbol{\eta}$ represents measurement noise.
The forward model may be linear or nonlinear and is typically non-injective, so recovering $\rvx$ from $\rvy$ is ill-posed and requires prior information about the unknown signal.

In a Bayesian formulation, this prior information is modeled by a density $p(\rvx)$ over the signal space, while the forward model induces the likelihood $p(\rvy \mid \rvx)$.
A reconstruction can then be obtained either by sampling from the posterior $p(\rvx \mid \rvy)$, e.g., by training conditional score-based diffusion models \citep{batzolis2021conditional,saharia2022image}, or through maximum-a-posteriori (MAP) estimation. 
The latter yields the minimization problem
\begin{equation}
\hat{\rvx} = \operatorname*{arg\,min}_{\rvx}
\left\{ -\log p(\rvy \mid \rvx) - \log p(\rvx) \right\} =
\operatorname*{arg\,min}_{\rvx}
\left\{ l(\rvy;\rvx) + s(\rvx) \right\},
\label{mapequation}
\end{equation}
where $l(\rvy;\rvx)=-\log p(\rvy\mid\rvx)$ measures consistency with the
forward model and $s(\rvx)=-\log p(\rvx)$ encodes the prior.
This formulation allows a single unconditional generative model $p(\rvx)$ to serve as a reusable prior across multiple inverse problems and forward models.
We refer to $l(\rvy;\rvx)$ and $s(\rvx)$ as the
\emph{forward-consistency} and \emph{prior-consistency} terms, respectively.
\footnote{A shorter version of this paper is accepted at the PRIGM Workshop at NeurIPS 2026.}


In this paper we revisit the flow-based PnP framework introduced by \citet{kumarpnpadmm} for solving the MAP problem in \eqref{mapequation}, which we refer to as PnP-CF (\emph{Plug-and-Play Consensus Flow}), and extend its application to nonlinear inverse problems, including those involving learned neural-network forward operators.
The algorithmic formulation of PnP-CF, recalled in section \ref{methods}, follows \citet{kumarpnpadmm}; our contributions lie in its extension to and evaluation on nonlinear and learned forward models.

PnP-CF decouples the forward-consistency and prior-consistency updates within an iterative scheme, similar to the PnP formulation of \citet{venkatakrishnan2013plug}, while realizing the prior-consistency update through a pretrained flow-matching denoiser, avoiding backpropagation through the generative prior.
However, integrating forward-consistency optimization with a generative prior such as flow matching poses a significant challenge in maintaining consensus between the two, since the forward-consistency step is driven by the measurements, whereas the flow prior is designed to operate along a structured family of intermediate probability paths.

PnP-CF addresses this challenge through two complementary mechanisms: (i) a decoupled ADMM-like (Alternating Direction Method of Multipliers \citep{boyd_distributed_2010}) variable-splitting structure with dual-variable updates and (ii) a time-adaptive pullback toward the intermediate flow probability path using complementary stochastic and deterministic components (section \ref{methods}).
We argue that these mechanisms stabilize the interaction between the forward-consistency solver and the flow-matching denoiser, thereby promoting consensus between the measurement-driven iterate and the learned flow trajectory (section \ref{pnpcf_analysis}).
While methods such as PnP-Flow \citep{martin2025pnpflow} and FLOWER \citep{pourya2026flower} use flow-based denoisers as priors in similar PnP frameworks, we demonstrate through a series of experiments that PnP-CF not only achieves strong performance on standard linear inverse problems but also extends naturally to nonlinear forward operators, broadening the applicability of flow-based inverse solvers beyond the predominantly linear settings considered in existing work (section \ref{baseline_comparison} and table \ref{tab:results}).
We further evaluate PnP-CF on segmentation-guided reconstruction, where the forward operator is a complex neural network that discards substantial image information, demonstrating that PnP-CF extends beyond analytical measurement models to highly non-injective learned operators (section \ref{segmentation} and table \ref{tab:segmentation}).
Relative to \citet{kumarpnpadmm}, we thus contribute a forward-consistency solver for nonlinear and learned operators, the first nonlinear evaluation of several flow-based baselines, and a segmentation-guided formulation with a highly non-injective neural forward operator.
Our code is available at \url{https://github.com/dibyanshu100/PnP_ConsensusFlow.git}.

%% file: background.tex
\section{Background}
\label{sec:bkg}

\subsection{Flow Matching and Straight Trajectories}
\label{fm_denoiser}

Flow matching \citep{lipman2023flowmatching} is a generative framework for learning a deterministic transformation that transports samples from a simple base distribution $p_0$, typically a standard Gaussian, to a target data distribution $p_1$. 
The central idea is to specify a continuous path of intermediate distributions $(p_t)_{t \in [0,1]}$ that interpolates between $p_0$ and $p_1$, and then to learn a time-dependent velocity field $\rvv_\vtheta$ whose flow $\psi$ follows this path.

A typical choice is to construct $p_t$ from straight-line path \cite{liu2022flowstraight} between samples $(\rvx_0, \rvx_1) \sim \pi$ drawn from a coupling $\pi \in \Gamma(p_0, p_1)$, i.e., a joint distribution with marginals $p_0$ and $p_1$ such as the independent coupling $\pi = p_0 \otimes p_1$, defined as
\begin{equation}
\rvx_t = (1 - t)\,\rvx_0 + t\,\rvx_1, \quad t \in [0,1],
\end{equation}
and $p_t$ is the implicit distribution of $\rvx_t$. 
This path moves from a noise sample at $t=0$ to a data sample at $t=1$ with constant conditional velocity $\rvx_1 - \rvx_0$. 

The flow associated with this probability path can be characterized by a flowmap $\psi_t$ generated by the marginal velocity field $\rvv_t(\rvx) := \mathbb{E}_{\pi}[\rvx_1 - \rvx_0 \mid \rvx_t = \rvx]$, solving
\begin{equation}
\frac{d}{dt} \psi_t(\rvx_0) = \rvv_t(\psi_t(\rvx_0)), \quad \psi_0(\rvx_0) = \rvx_0,
\end{equation}
such that $p_t = (\psi_t)_\# p_0$, and in particular $\psi_1$ maps $p_0$ to $p_1$.
In practice, $\rvv_t$ is unknown and we approximate it using a neural network $\rvv_\vtheta$. 
For straight-line trajectories, the per-pair target velocity is $\rvv_t^*(\rvx_t \mid \rvx_0, \rvx_1) = \rvx_1 - \rvx_0$; regressing onto this target recovers the marginal field $\rvv_t$ in expectation, leading to the conditional flow matching objective
\begin{equation}
\mathcal{L}_{\text{CFM}}(\vtheta) = \mathbb{E}_{(\rvx_0, \rvx_1) \sim \pi, t\sim \mathcal{U}(0,1)}
\left[\left\| \rvv_\vtheta\big((1 - t)\rvx_0 + t \rvx_1, t\big) - (\rvx_1 - \rvx_0) \right\|^2\right].
\end{equation}
For inference, the flow from time $t$ to $1$ can be written as $\psi_1(\rvx_t) = \rvx_t + \int_t^1 \rvv_\vtheta(\rvx_q, q)\, dq$.
This integral can be approximated using an explicit Euler step, yielding the one-step estimate
\begin{equation}
\psi_1(\rvx_t) \approx \tilde{\psi}_1(\rvx_t) = \rvx_t + (1 - t)\, \rvv_\vtheta(\rvx_t, t).
\label{one_step_denoiser}
\end{equation}
The resulting flow operator $\tilde{\psi}_1$ can be interpreted as a one-step, time-dependent denoiser, mapping $\rvx_t$ toward the data distribution at $t=1$. 
We adopt an operator-based view of PnP, using $\tilde{\psi}_1$ as the denoising operator, i.e., $\mathcal{D}_t(\rvx_t) := \tilde{\psi}_1(\rvx_t)$, and describe its role within the PnP-CF framework in section \ref{methods}.

\subsection{Plug-and-Play Algorithms}
\label{bkg:fbs}

\citet{venkatakrishnan2013plug} introduced the PnP algorithm, which addresses inverse problems by combining a forward-consistency term with a learned prior implemented as a denoising operator. 
A key feature of PnP is that the prior is trained independently of the inverse problem, typically as a standalone denoiser, and can therefore be reused across different forward operators. 
This decoupling enables the use of powerful learned priors without requiring problem-specific retraining.
Early PnP methods relied on classical or CNN-based denoisers \citep{Zhang_2017,ryu2019plugandplay}, while subsequent methods incorporated expressive generative priors including GANs \citep{bora2017compressed} and diffusion models \citep{zhu2023denoising,mardani2023variational,li2025decoupled,shrestha2026}.

Motivated by the strong empirical performance of flow matching, several recent works have adapted the PnP framework to use flow-matching models as learned priors.
\citet{martin2025pnpflow} propose PnP-Flow, which integrates a pretrained flow-matching denoiser into a Forward-Backward Splitting (FBS) scheme.
A related approach, FLOWER, proposed by \citet{pourya2026flower}, combines a flow-based denoising step with a proximal forward-consistency update and provides a Bayesian interpretation of the resulting procedure.
We collectively refer to these approaches as \emph{FBS-based} methods.
While these methods demonstrate the effectiveness of flow-matching priors within PnP formulations for relatively simple linear forward operators, the lack of consensus between measurement-driven updates and the evolving flow trajectory limits their effectiveness on more challenging problems (see section \ref{baseline_comparison}).
This motivates PnP-CF, which is specifically designed to maintain consensus between these components, enabling flow-based PnP algorithms to handle more complex forward models, including nonlinear and learned neural-network forward operators.
In the next section, we describe the methodological formulation of PnP-CF.



%% file: methods.tex
\section{Method}
\label{methods}

Akin to ADMM and similar in spirit to the original PnP framework of \citet{venkatakrishnan2013plug}, we decouple the forward-consistency $l(\rvy;\rvx)$ and prior-consistency $s(\rvx)$ terms in \eqref{mapequation} by introducing an auxiliary variable $\rvz$ and enforcing $\rvx=\rvz$.

The corresponding augmented Lagrangian is
\begin{equation}
\mathcal{L}_{\rho}(\rvx,\rvz,\rvu)
= l(\rvy;\rvx) + s(\rvz) + \frac{\rho}{2}\|\rvx-\rvz+\rvu\|_2^2,
\end{equation}
where $\rvu$ is the scaled dual variable and $\rho>0$ is a penalty parameter that controls the strength of the consensus constraint.
The resulting constrained problem is solved iteratively by alternating updates of the primal variables $\rvx,\rvz$ and the dual variable $\rvu$.
At each iteration, the primal variables are updated through two separate subproblems:
\begin{equation}
\label{eq:admm_updates}
\hat{\rvx} \leftarrow \underbrace{\arg\min_{\rvx} l(\rvy;\rvx) + \frac{\rho}{2}\|\rvx-\tilde{\rvx}\|_2^2}_{\mathcal{F}(\rvy,\tilde{\rvx}) \text{ forward-consistency operator}},
\qquad
\hat{\rvz} \leftarrow \underbrace{\arg\min_{\rvz} s(\rvz) + \frac{\rho}{2}\|\tilde{\rvz}-\rvz\|_2^2}_{\mathcal{D}(\tilde{\rvz}) \text{ denoising operator}},
\end{equation}
followed by the scaled dual update $\rvu \leftarrow \rvu + (\hat{\rvx}-\hat{\rvz})$.

The two operators in \eqref{eq:admm_updates} can be interpreted as proximal updates \citep{ParikhBoyd2014} centered at the anchor points $\tilde{\rvx}$ and $\tilde{\rvz}$, respectively. 
The forward-consistency operator $\mathcal{F}$ depends only on the forward model through $l(\rvy;\rvx)$, while $\mathcal{D}_t(\rvx_t):= \tilde{\psi}_1(\rvx_t)$ denotes the one-step flow denoiser introduced in section \ref{fm_denoiser}. 
In the following, we omit the time dependence and write $\mathcal{D}$ for notational simplicity.

\begin{figure}[t]
    \centering
    \includegraphics[width=0.75\linewidth]{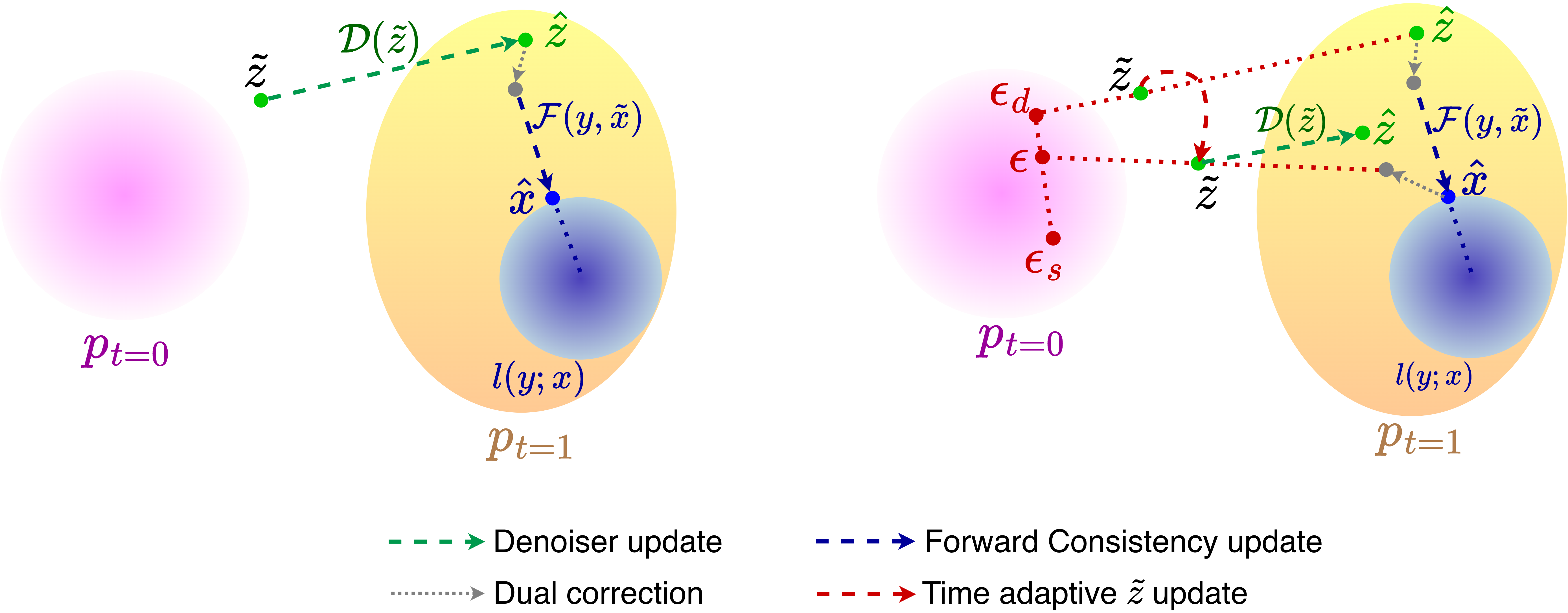}
\caption{Phases of PnP-CF. \emph{Left:} Starting from $\tilde\rvz$, the denoiser $\mathcal{D}(\tilde\rvz)$ produces $\hat\rvz$ (green). The dual correction (gray) shifts $\hat\rvz$ by $\rvu$, and the forward-consistency operator $\mathcal{F}(\rvy,\tilde\rvx)$ maps it to $\hat\rvx$ (blue), pulling it towards the region consistent with $\rvy$ under $l(\rvy;\rvx)$. \emph{Right:} The time-adaptive pullback constructs the next $\tilde\rvz$. The deterministic component $\rvepsilon_d$, computed from the previous $\tilde\rvz$ and $\hat\rvz$, combines with $\rvepsilon_s \sim p_{t=0}$ to form $\rvepsilon$. The dual-corrected point $\hat\rvx + \rvu$ (gray) and the $\rvepsilon$ together produces the new $\tilde\rvz$, which serves as the input to the denoiser at the next iteration.}
    \label{fig:projection}
\end{figure}

A critical design choice of the PnP algorithms is the construction of the anchor points $\tilde{\rvx}$ and $\tilde\rvz$ connecting the alternating updates.
We construct $\tilde{\rvx}$ using $\mathcal{D}$, the one-step flow denoiser, where we, with a slight abuse of notation, align the discretized ODE time-steps $t \in [0,1)$ with the main algorithm iterations $k = 0, \dots, K-1$.
Importantly and in contrast to the FBS-based methods, we further adjust this point through the dual variable $\rvu$ to form the consensus point $\tilde{\rvx} \leftarrow \mathcal{D}(\tilde\rvz) - \rvu$ serving as the input for the next iteration of the forward-consistency update (figure \ref{fig:projection} left).

Conceptually, the inputs $\tilde{\rvx}$ and outputs $\hat\rvx$ of the forward-consistency operator $\mathcal{F}$ live in the endpoint $p_{t=1}$ of the learned probability path $(p_t)_{t\in[0, 1]}$ of the flow-matching model.
However, as the proximal update $\hat\rvx = \mathcal{F}(\rvy, \tilde{\rvx})$ pushes the intermediate sample $\tilde{\rvx}$ towards the region consistent with the forward model, the output $\hat\rvx$ may end up in a low-probability region under the learned flow-matching prior (\emph{off-distribution}) and should be further aligned before applying the denoiser $\mathcal{D}$ in the subsequent iteration.

Since the denoiser $\mathcal{D}$ operates over samples drawn from intermediate distributions $(p_t)_{t < 1}$ we construct $\tilde\rvz$ by a time-adaptive pullback step along the flow-matching straight path, e.g. \citet{liu2022flowstraight}, $\tilde{\rvz} \leftarrow (1-t)\, {\rvepsilon} + t \, (\hat\rvx + {\rvu})$.
While similar pullback has been used in the FBS-based methods, our update differs in two critical aspects.
First, the ADMM-like structure allows to further stabilize the algorithm by enforcing consensus between the primal updates through the dual variable ${\rvu}$.
Second, we construct the latent sample ${\rvepsilon}$ from the flow base distribution $p_{t=0}$ as a time-adaptive combination of a stochastic $\rvepsilon_s$ (used in FBS-based methods), and a deterministic component $\rvepsilon_d$ (not used in FBS-based methods) computed from the pullback point of the previous iteration as shown in the right part of figure \ref{fig:projection}:
\begin{equation}
{\rvepsilon} \leftarrow
\frac{t\,\rvepsilon_d + (1 - t)\,\rvepsilon_s}
{\sqrt{t^2 + (1-t)^2}}, 
\quad 
\rvepsilon_s \sim p_{t=0} = \mathcal{N}(\boldsymbol{0}, \mI),
\quad
\rvepsilon_d \leftarrow \frac{\tilde{\rvz} - t\,\hat\rvz}{1 - t} \enspace ,
\label{eq:adaptive_projection}
\end{equation}
where we again use the simplified notation $t \in [0, 1)$ aligned with the main algorithm iterations $k = 0, \ldots, K-1$. 

The two components in \eqref{eq:adaptive_projection} play complementary roles.
The stochastic component $\rvepsilon_s$, independently sampled from the flow base distribution $p_{t=0}$, can move the iterate across different flow trajectories and therefore promotes exploration.
In contrast, the deterministic component $\rvepsilon_d$ inferred from the previous
pullback point preserves information about the trajectory already followed, keeping subsequent updates closer to the current flow trajectory.
The denominator $\sqrt{t^2 + (1-t)^2}$ in \eqref{eq:adaptive_projection} normalizes the linear combination of the two components to a unit variance and maintains a consistent scale across iterations.
This design follows a simple principle: at early iterations (small $t$), the algorithm emphasizes stochastic exploration via $\rvepsilon_s$, allowing the iterate to traverse different regions of the solution space.
As $t \to 1$, the weight progressively shifts toward $\rvepsilon_d$, encouraging refinement along a consistent flow trajectory aligned with the denoiser.
Overall, this method enables the algorithm to both recover from poor initializations and commit to a stable trajectory as it approaches a solution that satisfies the forward-consistency term.


Finally and complementary to the time-adaptive pullback, the dual update $\rvu \leftarrow \rvu + (\hat{\rvx}-\hat{\rvz})$ provides an explicit mechanism for enforcing consensus between the forward and prior-consistent solutions. 
The residual $\hat{\rvx}-\hat{\rvz}$ measures their disagreement, and its accumulation in $\rvu$ carries this information across iterations, so that inconsistencies are not discarded after a single update.
The dual variable then feeds this accumulated discrepancy back into both sides of the alternating scheme: it shifts the denoised solution through $\tilde{\rvx}=\hat{\rvz}-\rvu$ before the next forward-consistency update and adjusts the pullback endpoint through $\hat{\rvx}+\rvu$ before the subsequent prior-consistency step. 
In this way, $\rvu$ acts as a persistent correction signal that counteracts systematic drift between the two updates and progressively drives them toward a common solution. 
Together with the time-adaptive pullback, this provides PnP-CF with a mechanism for maintaining agreement between forward-consistency and the evolving flow trajectory throughout optimization.

In summary, relative to the FBS-based formulations, PnP-CF introduces two mechanisms for maintaining consistency throughout the iterative reconstruction: (i) the dual variable $\rvu$, which accumulates disagreement between the forward-consistency and denoising updates and feeds it back into subsequent iterations, and (ii) the deterministic component $\rvepsilon_d$ in the time-adaptive pullback, which retains information from the previous pullback and promotes refinement along a consistent flow trajectory.
Together, these mechanisms provide explicit cross-iteration consensus and trajectory continuity, while allowing forward and prior-consistency subproblem to be optimized more flexibly. 
This makes the formulation applicable to more involved forward-consistency problems, including nonlinear forward operators (section \ref{experiments}).
We further analyze these design choices in section \ref{pnpcf_analysis}.
The complete PnP-CF procedure is summarized in algorithm \ref{algo:pnpcf}, with the iteration index $k$ reintroduced to make the update sequence explicit.


\begin{algorithm}[t]
\caption{PnP-CF with time-adaptive pullback (adapted from \cite{kumarpnpadmm})}
\label{algo:pnpcf}
\begin{algorithmic}[1]
\STATE \textbf{Input:} Flow denoiser $\mathcal{D}_t$, observation $\rvy$, forward-consistency $l(\rvy;\rvx)$, penalty parameter $\rho_k$, number of iterations $K \in \mathbb{N}$, time steps $t_k=k/K$
\STATE \textbf{Initialize:} $\tilde{\rvx}^{\,0}\sim\mathcal{N}(\boldsymbol{0},\mI)$, $\rvu^0=\boldsymbol{0}$, $\rvepsilon^0=\tilde{\rvx}^{\,0}$
\FOR{$k=0,\dots,K-1$}
\STATE $\hat{\rvx}^{\,k+1}=\arg\min_{\rvx}\; l(\rvy;\rvx)+\frac{\rho_k}{2}\|\rvx-\tilde{\rvx}^{\,k}\|_2^2$ \hfill $\triangleright$ forward-consistency
\STATE $\tilde{\rvz}^{\,k}=(1-t_k)\rvepsilon^k+t_k(\hat{\rvx}^{\,k+1}+\rvu^k)$ \hfill $\triangleright$ time-adaptive pullback
\STATE $\hat{\rvz}^{\,k+1}=\mathcal{D}_{t_k}(\tilde{\rvz}^{\,k})$ \hfill $\triangleright$ prior-consistency
\STATE $\rvepsilon_d^{\,k}=\frac{\tilde{\rvz}^{\,k}-t_k\hat{\rvz}^{\,k+1}}{1-t_k}, \qquad \rvepsilon_s^{\,k}\sim\mathcal{N}(\boldsymbol{0},\mI)$ \hfill $\triangleright$ deterministic and stochastic components
\STATE $\rvepsilon^{k+1}=\frac{t_k\rvepsilon_d^{\,k}+(1-t_k)\rvepsilon_s^{\,k}}{\sqrt{t_k^2+(1-t_k)^2}}$ \hfill $\triangleright$ adaptive latent update
\STATE $\rvu^{k+1}=\rvu^k+\hat{\rvx}^{\,k+1}-\hat{\rvz}^{\,k+1}$ \hfill $\triangleright$ dual update
\STATE $\tilde{\rvx}^{\,k+1}=\hat{\rvz}^{\,k+1}-\rvu^{k+1}$ \hfill $\triangleright$ consensus anchor
\ENDFOR
\STATE \textbf{return} $\hat{\rvz}^{\,K}$
\end{algorithmic}
\end{algorithm}

%% file: experiments.tex
\section{Experiments and Results}
\label{experiments}

In this section, we evaluate PnP-CF on inverse problems defined by a range of linear and nonlinear forward models, $\rvy = \mathcal{A}(\rvx) + \boldsymbol{\eta}$, where $\boldsymbol{\eta} \sim \mathcal{N}(\boldsymbol{0}, \sigma^2 \mI)$ and $\mathcal{A}$ denotes either a linear or nonlinear operator.
Under this additive Gaussian-noise model, the forward-consistency term in \eqref{mapequation} becomes
$l(\rvy; \rvx) = \frac{1}{2\sigma^2} \|\mathcal{A}(\rvx) - \rvy\|_2^2$.
For linear forward operators, the corresponding forward-consistency proximal $\mathcal{F}(\rvy,\tilde{\rvx})$ reduces to a convex quadratic problem, which we solve efficiently using conjugate gradients \citep{hestenes1952methods}.
For nonlinear forward operators, the resulting forward-consistency subproblem is generally nonquadratic and does not admit a simple solution. 
We therefore solve it approximately using gradient-based optimization through the differentiable forward operator $\mathcal{A}(\rvx)$ (see appendix \ref{appendix:forward_models} for details).

For our experiments, we evaluate on randomly selected test images from CelebA \citep{yang2015facialparts} ($128\times128$) and AFHQ-Cat \citep{choi2020starganv} ($256\times256$), using pretrained flow-matching models from \cite{martin2025pnpflow}, and report PSNR and SSIM for reconstruction fidelity and LPIPS for perceptual quality. 
We consider three linear inverse problems: box inpainting, random inpainting, and super-resolution.
For box inpainting, the forward operator $\mathcal{A}$ masks a centered $60\times60$ region for CelebA and a $120\times120$ region for AFHQ-Cat.
For random inpainting, $\mathcal{A}$ masks $85\%$ of the pixels uniformly at random.
For super-resolution, $\mathcal{A}$ downsamples the image by factors of $\times4$ and $\times8$ for CelebA and AFHQ-Cat, respectively.
We add Gaussian measurement noise with $\sigma=0.05$ for box inpainting and super-resolution, and $\sigma=0.01$ for random inpainting.

We additionally consider three nonlinear inverse problems: nonlinear deblurring, phase retrieval, and HDR reconstruction. 
For nonlinear deblurring, the forward operator is a pretrained KernelWizard network using the GOPRO-wVAE model \citep{tran2021blurencoding}, which maps a sharp image to a blurred one via a nonlinear network forward pass conditioned on a latent blur code.
For phase retrieval, the forward operator is $\mathcal{A}(\rvx)=|\mathcal{T}(\mathcal{P}(\rvx))|$, where $\mathcal{P}$ denotes symmetric zero-padding and $\mathcal{T}$ is the orthonormal two-dimensional Fourier transform; we use an oversampling factor of $2$.
For HDR reconstruction, we use the nonlinear saturation operator $\mathcal{A}(\rvx)=\operatorname{clip}(2\rvx,-1,1)$. Gaussian measurement noise is added with $\sigma=0.05$ for nonlinear deblurring and HDR reconstruction and $\sigma=0.01$ for phase retrieval.

The main hyperparameter of PnP-CF is the penalty $\rho$, which controls the consensus strength between the forward-consistency and prior updates.
In contrast to \citet{kumarpnpadmm}, who use a constant penalty and a linear time grid, we introduce the schedules $\rho_k=\rho/(1-t_k)^\beta$ and $t_k=(k/K)^\gamma$, where $\beta$ controls the growth of the penalty and $\gamma$ controls the progression along the flow trajectory.
Since the pullback step weights $\epsilon_d$ by $t_k$, choosing $\gamma<1$ shifts the emphasis toward the deterministic component $\epsilon_d$ earlier, whereas $\gamma>1$ retains the stochastic component $\epsilon_s$ for longer (see appendix \ref{app:role_hyperparams}).
Following \citet{martin2025pnpflow}, we average the outputs of several denoiser calls per iteration.
The hyperparameter values used in our experiments are reported in appendix \ref{app:hyperparams_intro}.

\subsection{Analysis of PnP-CF}
\label{pnpcf_analysis}

\begin{figure}[ht]
    \centering

    \begin{minipage}[t]{0.58\linewidth}
        \centering
        \includegraphics[
            width=\linewidth,
            height=0.24\textheight,
            keepaspectratio
        ]{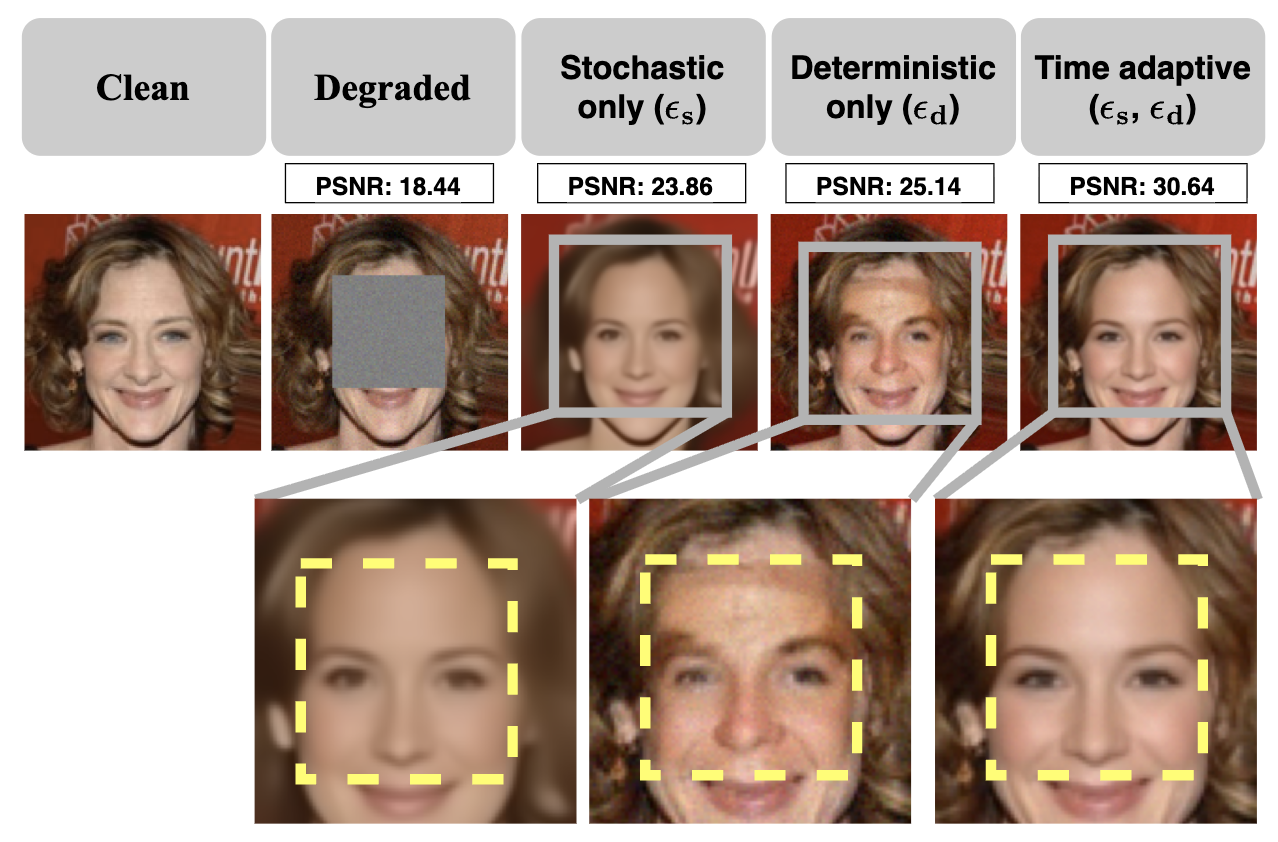}
    \end{minipage}
    \hfill
    \begin{minipage}[t]{0.41\linewidth}
        \centering
        \includegraphics[
            width=\linewidth,
            height=0.24\textheight,
            keepaspectratio
        ]{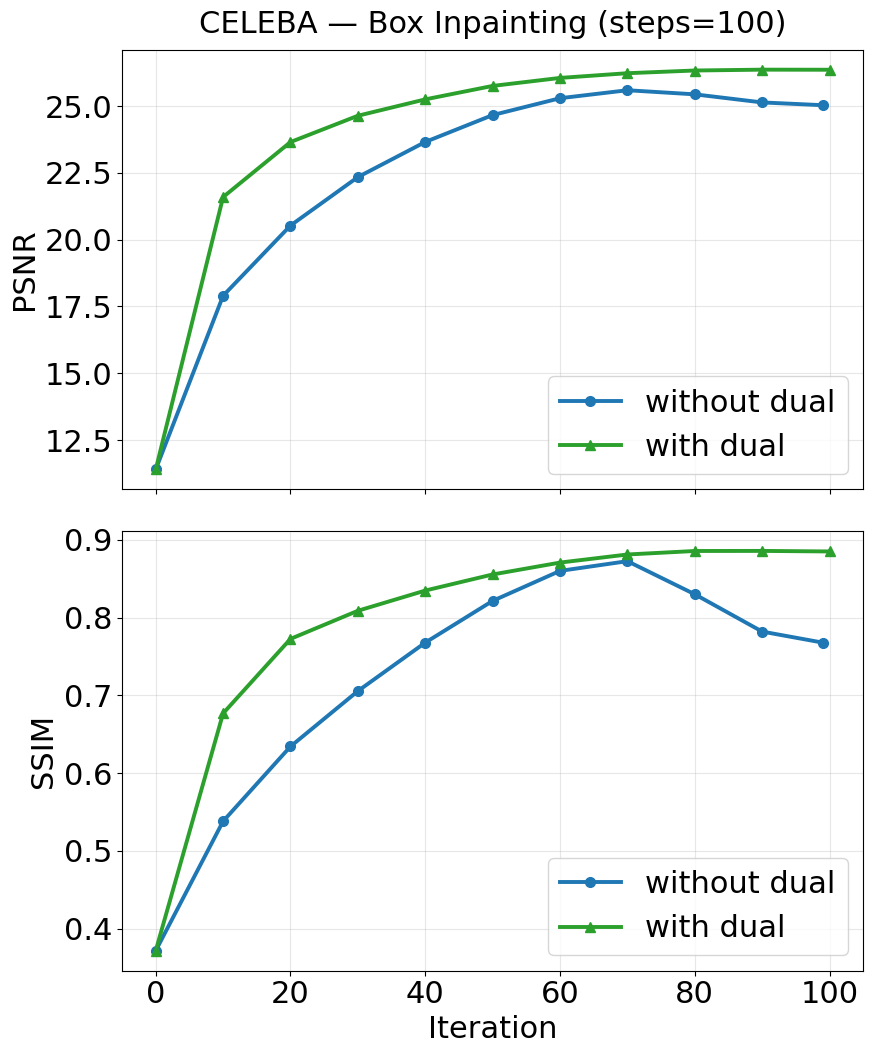}
    \end{minipage}
    \caption{Ablation of PnP-CF update strategies on the box-inpainting inverse problem.
Left: qualitative comparison of stochastic-only, deterministic-only, and time-adaptive pullback strategies.
Right: effect of the dual variable $\rvu$ on the forward-consistency and denoising updates, shown through PSNR and SSIM trajectories over iterations, averaged over 100 images.}
    \label{fig:ablation_celeba}
\end{figure}

We first build intuition for the components of PnP-CF algorithm through a qualitative comparison on a linear box-inpainting problem.
In the left side of figure \ref{fig:ablation_celeba} we show the effects of the stochastic and deterministic components of the pullback step in \eqref{eq:adaptive_projection}.
Using only the stochastic component $\rvepsilon_s$ (``stochastic-only'') causes the iterates to oscillate across different randomly sampled flow trajectories. Averaging these inconsistent samples yields blurry reconstructions.
Using only the deterministic component $\rvepsilon_d$ (``deterministic-only''), the algorithm becomes highly dependent on the random initial flow state, restricting exploration to a single solution path that may be difficult to align with the forward-model constraints.
Finally, the proposed time-adaptive combination promotes exploration in early iterations and gradually transitions to trajectory-consistent updates, yielding reconstructions that are both sharp and semantically consistent.

In the right side of figure \ref{fig:ablation_celeba}, we show the effects of the ADMM-like structure of PnP-CF method on the same box-inpainting task, comparing \emph{with} and \emph{without} the dual variable $\rvu$ updates. 
Omitting $\rvu$ reduces the algorithm to FBS-style approach, with the time-adaptive pullback retained to isolate the effect of the dual update.
The graphs report the PSNR and SSIM at each iteration, averaged over 100 samples.
Without $\rvu$, the forward-consistency and denoising updates are solved independently at each iteration, resulting in slower convergence and lower reconstruction quality in terms of both PSNR and SSIM.
With $\rvu$, the accumulated discrepancy progressively aligns the two subproblems, producing a faster and more stable improvement in reconstruction quality.
We show a more detailed ablation in appendix \ref{ablation}.

\begin{figure}[h]
    \centering
    \includegraphics[width=\linewidth]{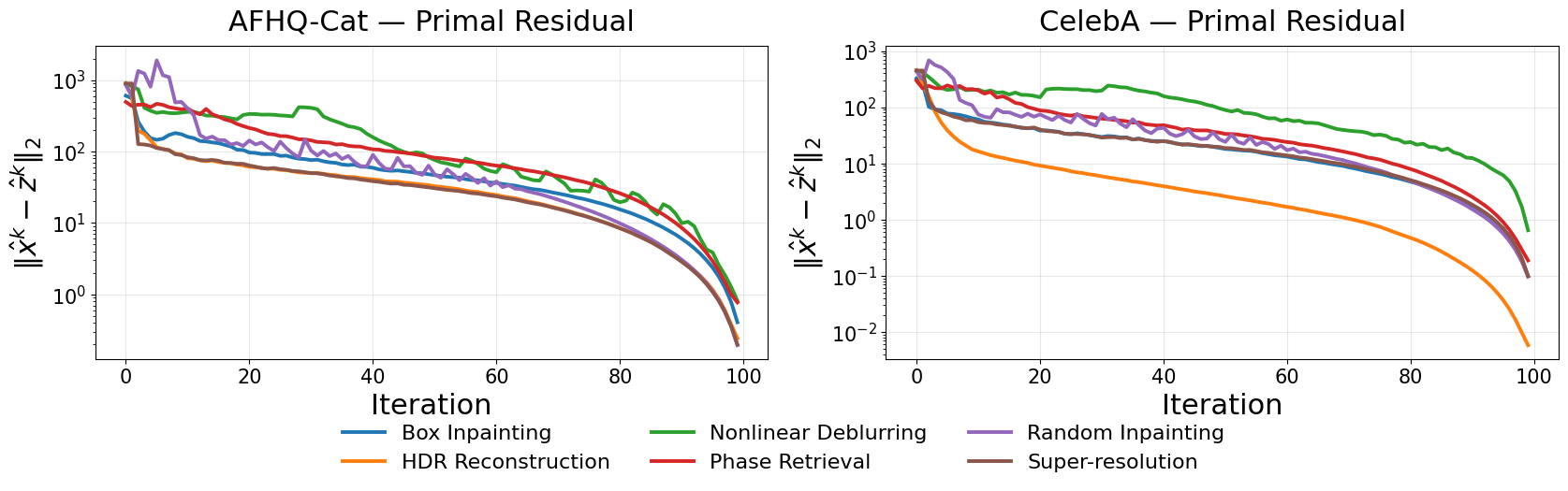}
    \caption{Evolution of the primal residual $\|\hat\rvx^k-\hat\rvz^k\|_2$ across linear and nonlinear inverse problems on AFHQ-Cat (left) and CelebA (right).}
    \label{fig:primal_res}
\end{figure}

We next examine the convergence behavior of PnP-CF through the primal residual
$r_{\mathrm{pri}}^k := \|\hat\rvx^k-\hat\rvz^k\|_2$.
As shown in figure \ref{fig:primal_res}, the residual decreases substantially across all linear and nonlinear tasks on both datasets, indicating progressive consensus between the forward and prior-consistency variables.
Since the prior update is implemented by a learned denoiser rather than an exact proximal operator, classical ADMM convergence guarantees \citep{boyd_distributed_2010} do not directly apply; we therefore interpret these results as empirical evidence of convergence.

\input{results_table.tex}

\begin{figure}[t]
    \centering
    \includegraphics[width=\linewidth]{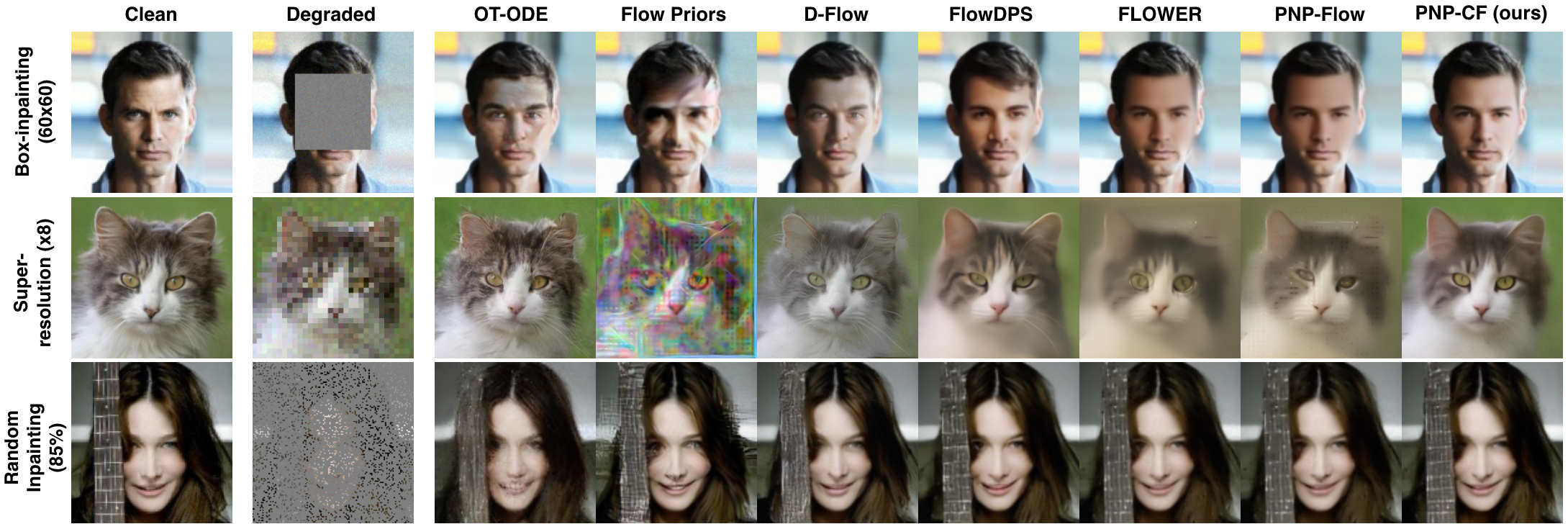}
    \caption{Comparison of image restoration methods on box inpainting, super-resolution, and random inpainting, with additional qualitative results provided in appendix \ref{appendix:linear}. Our method yields visually sharp reconstructions while preserving structural consistency across tasks.}
    \label{fig:results_linear}
\end{figure}

\subsection{Comparison with Baseline Methods}
\label{baseline_comparison}


We compare against six flow-based inverse solvers: OT-ODE \citep{pokle2024training}, D-Flow \citep{benhamu2024dflow}, Flow Priors \citep{zhang2025flowpriors}, PnP-Flow \citep{martin2025pnpflow}, FLOWER \citep{pourya2026flower}, and FlowDPS \citep{kim2025flowdps}.
Among these, only D-Flow reports results on nonlinear experiments.
We extend PnP-Flow and FlowDPS to nonlinear forward operators ourselves (appendix \ref{app:baseline_adaptation}), providing, to our knowledge, their first nonlinear evaluation.
OT-ODE and Flow Priors are derived specifically for linear operators and provide no mechanism for nonlinear measurements, while FLOWER only sketches a theoretical nonlinear extension via Langevin sampling without implementation or evaluation; we therefore exclude these three methods from the nonlinear comparison.
For linear inverse problems, we use stronger degradations than prior baseline evaluations to assess performance under more challenging measurement conditions.
Since the reported results are not directly comparable, we rerun all baselines on our test instances using their released implementations and recommended hyperparameter settings.
For completeness, we also report results under the milder settings of \citet{martin2025pnpflow,pourya2026flower} in appendix \ref{appendix:baselines}.

Table~\ref{tab:results} reports quantitative results averaged over 100 test images.
PnP-CF achieves the best performance on most tasks and metrics, with the largest margins on nonlinear problems.
FBS-based PnP methods and FlowDPS remain competitive on inpainting but degrade substantially on super-resolution and nonlinear tasks.
D-Flow achieves comparable performance on nonlinear deblurring, but is the least computationally efficient method overall.

Figure \ref{fig:results_linear} and right side of figure \ref{fig:results_nonlinear} show qualitative comparisons for the linear and nonlinear inverse problems, respectively.
For the linear tasks, several baselines recover the overall image structure but tend to lose fine details or introduce visible artifacts, particularly in box inpainting and super-resolution, whereas PnP-CF produces sharper reconstructions while better preserving structural consistency. 
This advantage extends to the nonlinear problems, where PnP-CF maintains consistently high reconstruction quality, while FBS-based methods show more pronounced degradation.

\begin{wraptable}{r}{0.35\linewidth}
\centering
\caption{Per batch runtime and peak GPU memory (HDR recon. - CelebA)}
\label{tab:efficiency-main}
\small
\setlength{\tabcolsep}{2pt}
\begin{tabular}{lcc}
\toprule
Method & Time (s) & Mem (GB) \\
\midrule
PnP-CF (ours) & 17.62 & \textbf{0.30} \\
PnP-Flow      & 16.40 & \textbf{0.30} \\
D-Flow        & 71.42 & 13.34 \\
FlowDPS       & \textbf{11.14} & 0.32 \\
\bottomrule
\end{tabular}
\end{wraptable}

Table~\ref{tab:efficiency-main} reports per-batch runtime and peak GPU memory on HDR reconstruction for CelebA, using a batch size of 4 and averaging over 25 batches.
PnP-CF remains computationally efficient, incurring only a small runtime overhead over PnP-Flow while maintaining a comparable peak memory footprint.
This efficiency comes directly from the design of PnP-CF: the flow matching denoiser is used only for the forward prior update and is never differentiated through.
Consequently, the additional nonlinear forward-consistency steps require gradients only through the forward operator, without storing or backpropagating through the flow trajectory.
This leads to a substantial computational advantage over trajectory-optimization methods such as D-Flow, which repeatedly integrates the ODE and backpropagates through the resulting trajectory while optimizing the initial noise.
Thus, PnP-CF extends flow-based reconstruction to complex nonlinear forward models while retaining the lightweight computational profile of PnP-style methods.
A more detailed analysis of computational complexity, together with runtime and memory results across inverse problems, is provided in appendix \ref{appendix:computational}.

\subsection{Segmentation-guided reconstruction}
\label{segmentation}

\begin{figure*}[t]
    \centering
    \begin{minipage}[c]{0.35\textwidth}
        \centering
        \includegraphics[width=\linewidth]{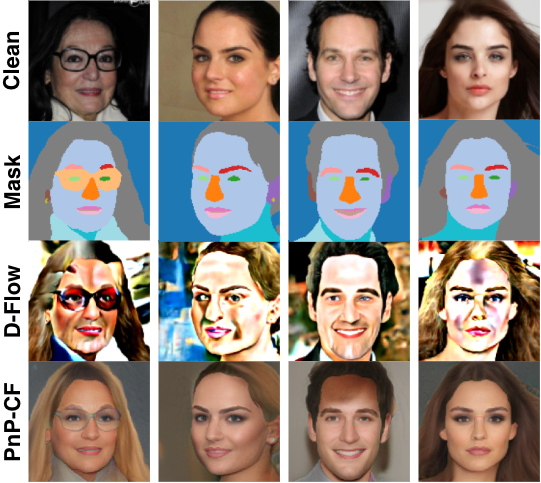}
    \end{minipage}
    \hfill
    \begin{minipage}[c]{0.62\textwidth}
        \centering
        \includegraphics[width=\linewidth]{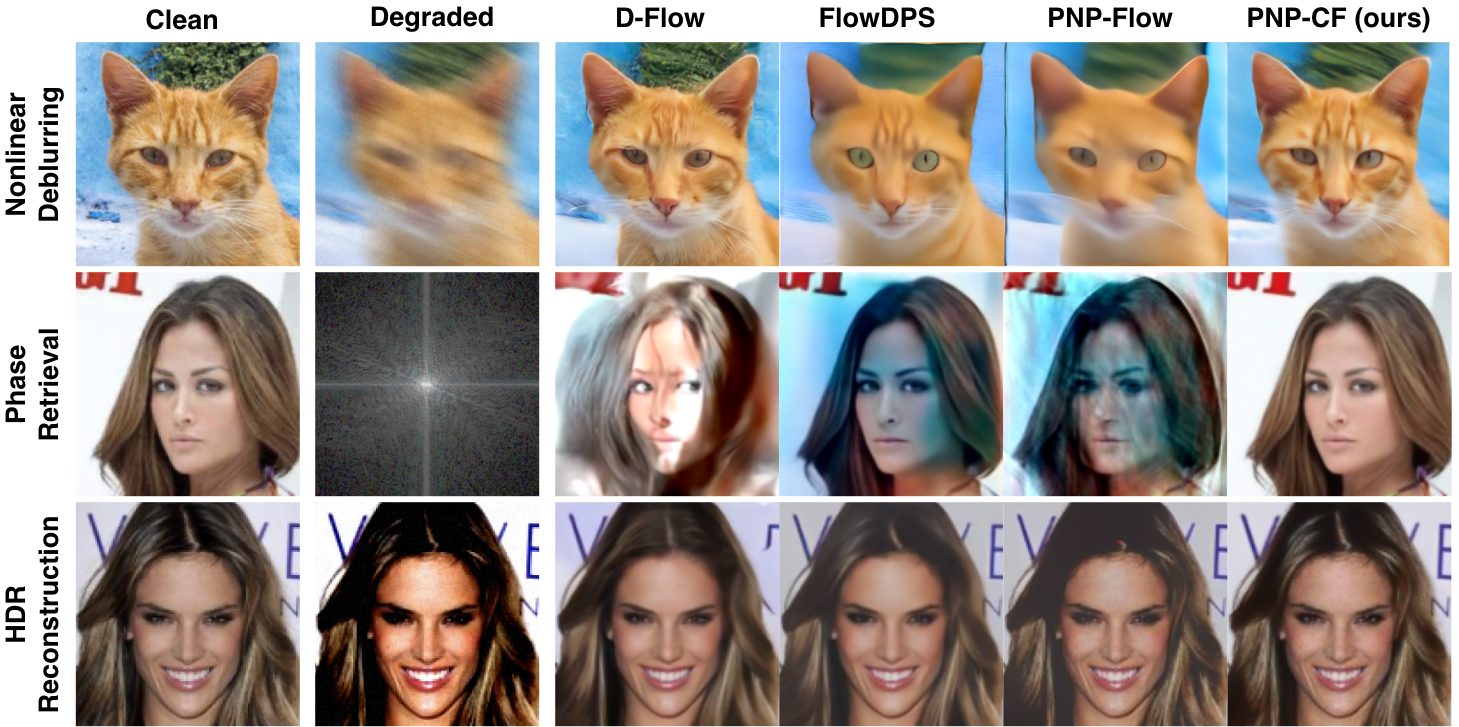}
    \end{minipage}

    \caption{
        Qualitative results. \textbf{Left:} Segmentation-guided reconstruction results for D-Flow and PnP-CF.
        \textbf{Right:} Qualitative reconstruction on nonlinear inverse problems, including HDR reconstruction,
        phase retrieval, and nonlinear deblurring with additional results in appendix \ref{appendix:nonlinear}.
    }
    \label{fig:results_nonlinear}
\end{figure*}

To further test the flexibility of PnP-CF beyond inverse problems with an additive Gaussian measurement model, we consider segmentation-guided reconstruction on CelebA-HQ \citep{karras2018progressive}.
This setting is substantially more challenging: the forward operator is a pretrained segmentation network, and is highly non-injective, since many visually distinct images can produce the same semantic segmentation mask.

Specifically, $\mathcal{A}(\rvx)$ is a pretrained segmentation network that outputs, for each pixel, a distribution over $C=19$ semantic classes.
Since these observations are categorical rather than corrupted by additive Gaussian noise, we replace the squared-error term with the corresponding negative log-likelihood of a categorical model, i.e., the cross-entropy loss
$l(\rvy;\rvx) = -\sum_i \sum_{c=1}^{C} \rvy_{i,c}\log \mathcal{A}(\rvx)_{i,c}$,
where $i$ indexes pixels and $\rvy$ denotes the target segmentation mask.
Similar to the other nonlinear forward operators, we approximate the corresponding proximal update using gradient-based optimization of the resulting forward-consistency objective.

\begin{wraptable}{r}{0.42\linewidth}
\centering
\caption{Segmentation-guided reconstruction results.}
\label{tab:segmentation}
\small
\begin{tabular}{lccc}
\toprule
\textbf{Method} & \textbf{LPIPS}$\downarrow$ & \textbf{SSIM}$\uparrow$ & \textbf{mIoU}$\uparrow$ \\
\midrule
PnP-CF  & \best{0.27} & \best{0.41} & \best{0.99} \\
PnP-Flow & 0.35 & 0.28 & 0.37 \\
D-Flow  & 0.38 & 0.11 & 0.98 \\
FlowDPS & 0.36 & 0.28 & 0.33 \\
\bottomrule
\end{tabular}
\end{wraptable}

The ambiguity of the segmentation mask makes pixel-wise reconstruction metrics difficult to interpret.
We therefore omit PSNR and instead report LPIPS and SSIM as perceptual and structural measures relative to the reference image, and additionally report mIoU, which measures consistency with the observed segmentation by comparing the segmentation mask predicted from the reconstructed image with the target segmentation mask.
Table \ref{tab:segmentation} shows that PnP-CF stands out by achieving both high mIoU and reconstruction quality.
Among the baselines, D-Flow is the only method to attain a similarly high mIoU, indicating strong adherence to the segmentation constraint, but it performs poorly in reconstruction quality as shown in figure \ref{fig:results_nonlinear} (left), where we compare PnP-CF with D-Flow. 
PnP-Flow and FlowDPS achieve substantially lower mIoU and fail to reliably satisfy the target segmentation constraint. 
We therefore omit them from the main qualitative figure and provide their reconstructions in appendix \ref{appendix:segmentation_additional}.
Qualitatively, PnP-Flow exhibits severe mode collapse, whereas FlowDPS produces plausible but weakly conditioned samples with limited correspondence to the target segmentation.

Overall, segmentation-guided reconstruction demonstrates a key strength of PnP-CF: its ability to handle complex learned forward operators while preserving both measurement consistency and reconstruction quality. 
Among the evaluated methods, PnP-CF is the only one to achieve this balance for the segmentation operator, showing that the framework extends naturally beyond standard linear and nonlinear inverse problems to substantially more complex neural forward models.

%% file: results_table.tex
\begin{table*}[t]
\centering

\caption{Reconstruction quality metrics for \textcolor{blue}{linear} (left panel) and \textcolor{darkgreen}{nonlinear} (right panel)
inverse problems on CelebA and AFHQ-Cat. Best results are shown in \textbf{bold}.}

\label{tab:results}

\begin{minipage}[t]{0.49\linewidth}
\centering
\setlength{\tabcolsep}{3pt}
\renewcommand{\arraystretch}{1.15}
\begin{adjustbox}{width=\linewidth}
\begin{tabular}{c l ccc ccc}
\toprule
& & \multicolumn{3}{c}{CelebA} & \multicolumn{3}{c}{AFHQ-Cat} \\
\cmidrule(lr){3-5} \cmidrule(lr){6-8}
Task & Method & PSNR$\uparrow$ & LPIPS$\downarrow$ & SSIM$\uparrow$
              & PSNR$\uparrow$ & LPIPS$\downarrow$ & SSIM$\uparrow$ \\
\midrule

\multirow{7}{*}{\rotatebox[origin=c]{90}{\textcolor{blue}{\makecell{Box\\Inpainting}}}}
& PnP-CF (ours)         & \best{26.36} & \best{0.05} & \best{0.89} & \best{22.65} & \best{0.13} & 0.82 \\
& PnP-Flow               & 24.75 & 0.08 & 0.88 & 21.41 & 0.21 & 0.82 \\
& FLOWER                 & 25.12 & 0.06 & 0.88 & 21.43 & 0.16 & \best{0.83} \\
& FlowDPS                & 24.22 & 0.08 & 0.86 & 19.41 & 0.21 & 0.79 \\
& D-Flow                & 24.60 & 0.08 & 0.80 & 21.68 & 0.19 & 0.71 \\
& Flow Priors            & 21.98 & 0.11 & 0.81 & 20.24 & 0.20 & 0.76 \\
& OT-ODE                & 24.07 & 0.07 & 0.85 & 19.45 & 0.16 & 0.78 \\

\midrule

\multirow{7}{*}{\rotatebox[origin=c]{90}{\textcolor{blue}{\makecell{Super\\Resolution}}}}
& PnP-CF (ours)         & \best{26.45} & 0.08 & \best{0.79} & \best{23.72} & 0.27 & \best{0.66} \\
& PnP-Flow               & 19.97 & 0.22 & 0.61 & 19.74 & 0.38 & 0.55 \\
& FLOWER                 & 22.63 & 0.15 & 0.70 & 19.30 & 0.45 & 0.55 \\
& FlowDPS                & 24.44 & 0.10 & 0.76 & 20.59 & 0.41 & 0.60 \\
& D-Flow                & 24.13 & 0.09 & 0.70 & 20.84 & 0.29 & 0.48 \\
& Flow Priors            & 23.54 & 0.19 & 0.56 & 17.77 & 0.55 & 0.35 \\
& OT-ODE                & 25.33 & \best{0.06} & 0.78 & 22.47 & \best{0.20} & 0.58 \\

\midrule

\multirow{7}{*}{\rotatebox[origin=c]{90}{\textcolor{blue}{\makecell{Random\\Inpainting}}}}
& PnP-CF (ours)         & \best{29.93} & \best{0.04} & \best{0.89} & \best{30.72} & \best{0.07} & \best{0.85} \\
& PnP-Flow              & 29.04 & 0.05 & 0.87 & 30.33 & 0.10 & \best{0.85} \\
& FLOWER                 & 29.04 & 0.05 & \best{0.89} & 30.39 & 0.10 & \best{0.85} \\
& FlowDPS                & 28.51 & 0.05 & 0.87 & 29.72 & 0.12 & 0.83 \\
& D-Flow                & 28.84 & \best{0.04} & 0.87 & 29.33 & 0.09 & 0.81 \\
& Flow Priors            & 25.65 & 0.08 & 0.82 & 28.37 & 0.11 & 0.82 \\
& OT-ODE                & 24.30 & 0.12 & 0.75 & 26.26 & 0.16 & 0.74 \\

\bottomrule
\end{tabular}
\end{adjustbox}
\end{minipage}%
\hfill
\begin{minipage}[t]{0.49\linewidth}
\centering
\setlength{\tabcolsep}{3pt}
\renewcommand{\arraystretch}{1.15}
\begin{adjustbox}{width=\linewidth}
\begin{tabular}{c l ccc ccc}
\toprule
& & \multicolumn{3}{c}{CelebA} & \multicolumn{3}{c}{AFHQ-Cat} \\
\cmidrule(lr){3-5} \cmidrule(lr){6-8}
Task & Method & PSNR$\uparrow$ & LPIPS$\downarrow$ & SSIM$\uparrow$
              & PSNR$\uparrow$ & LPIPS$\downarrow$ & SSIM$\uparrow$ \\
\midrule

\multirow{6}{*}{\rotatebox[origin=c]{90}{\textcolor{darkgreen}{\makecell{HDR\\Reconstruction}}}}
& PnP-CF (ours)         & \best{27.59} & \best{0.04} & \best{0.85} & \best{26.29} & \best{0.08} & \best{0.86} \\
& PnP-Flow               & 22.24 & 0.07 & 0.81 & 24.77 & 0.12 & 0.84 \\
& FLOWER                & -- & -- & -- & -- & -- & -- \\
& FlowDPS                & 22.40 & 0.10 & 0.80 & 25.25 & 0.18 & 0.81 \\
& D-Flow                & 24.16 & 0.09 & 0.77 & 25.59 & 0.16 & 0.75 \\
& Flow Priors            & -- & -- & -- & -- & -- & -- \\
& OT-ODE                & -- & -- & -- & -- & -- & -- \\

\midrule

\multirow{6}{*}{\rotatebox[origin=c]{90}{\textcolor{darkgreen}{\makecell{Phase\\Retrieval}}}}
& PnP-CF (ours)         & \best{30.33} & \best{0.11} & \best{0.80} & \best{25.27} & \best{0.22} & \best{0.70} \\
& PnP-Flow               & 22.57 & 0.22 & 0.62 & 16.19 & 0.45 & 0.41 \\
& FLOWER                & -- & -- & -- & -- & -- & -- \\
& FlowDPS                & 16.29 & 0.37 & 0.47 & 12.14 & 0.63 & 0.36 \\
& D-Flow                & 18.03 & 0.28 & 0.49 & 11.41 & 0.59 & 0.28 \\
& Flow Priors            & -- & -- & -- & -- & -- & -- \\
& OT-ODE                & -- & -- & -- & -- & -- & -- \\

\midrule

\multirow{6}{*}{\rotatebox[origin=c]{90}{\textcolor{darkgreen}{\makecell{Nonlinear\\Deblurring}}}}
& PnP-CF (ours)         & \best{27.09} & \best{0.09} & \best{0.79} & 22.28 & 0.29 & \best{0.58} \\
& PnP-Flow              & 25.43 & 0.21 & 0.72 & 20.60 & 0.48 & 0.54 \\
& FLOWER                & -- & -- & -- & -- & -- & -- \\
& FlowDPS                & 25.55 & 0.14 & 0.74 & 20.41 & 0.42 & 0.54 \\
& D-Flow                & 27.02 & \best{0.09} & 0.75 & \best{23.45} & \best{0.20} & 0.56 \\
& Flow Priors            & -- & -- & -- & -- & -- & -- \\
& OT-ODE                & -- & -- & -- & -- & -- & -- \\
\bottomrule
\end{tabular}
\end{adjustbox}
\end{minipage}

\end{table*}

%% file: appendix_main.tex
\input{appendix_conjugate_gradient.tex}
\input{appendix_experiments.tex}
\input{ablation.tex}

\clearpage

\section{Additional results on linear tasks with Gaussian noise}
\label{appendix:linear}

\begin{figure}[h]
    \centering
    \includegraphics[width=\textwidth,height=\textheight,keepaspectratio]{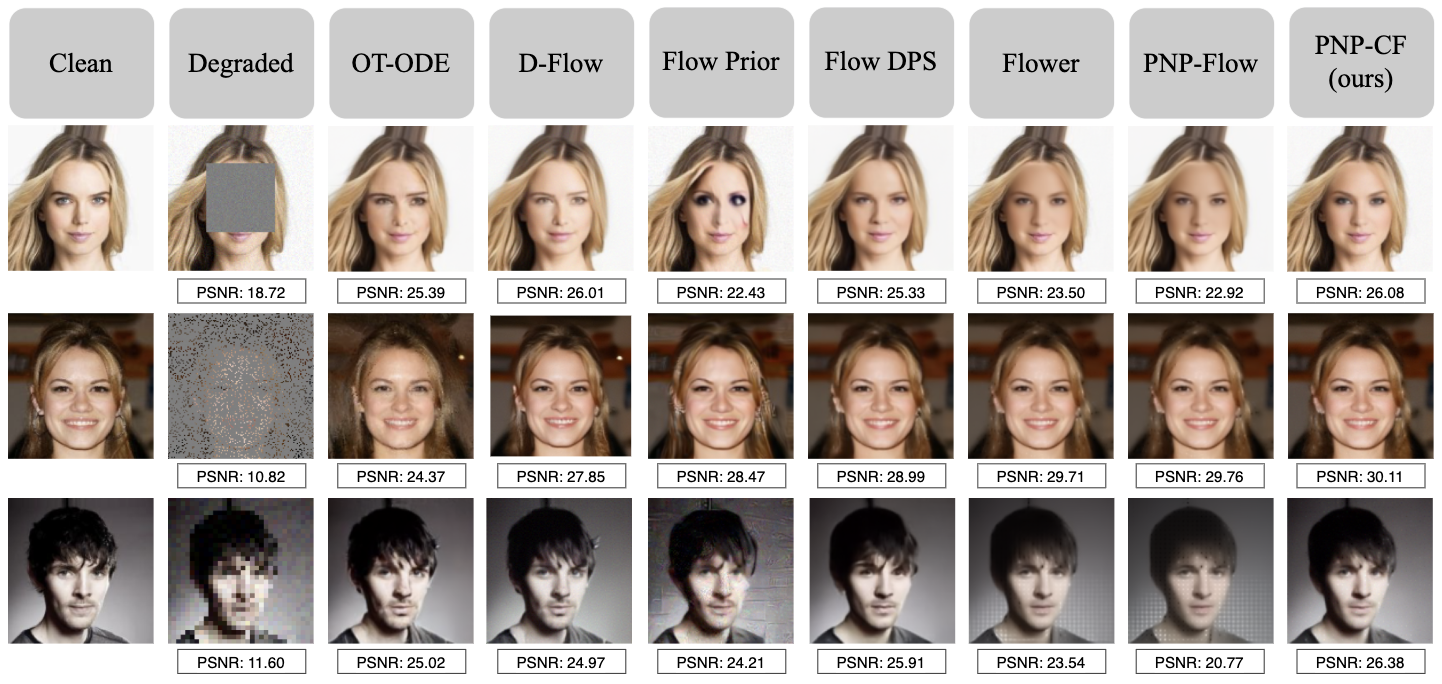}
    \caption{Additional qualitative comparisons on the CelebA dataset across linear inverse tasks, including 
    box inpainting (1st row), random inpainting (2nd row) and super-resolution (3rd row).
    }
\end{figure}

\begin{figure}[h]
    \centering
    \includegraphics[width=\textwidth,height=\textheight,keepaspectratio]{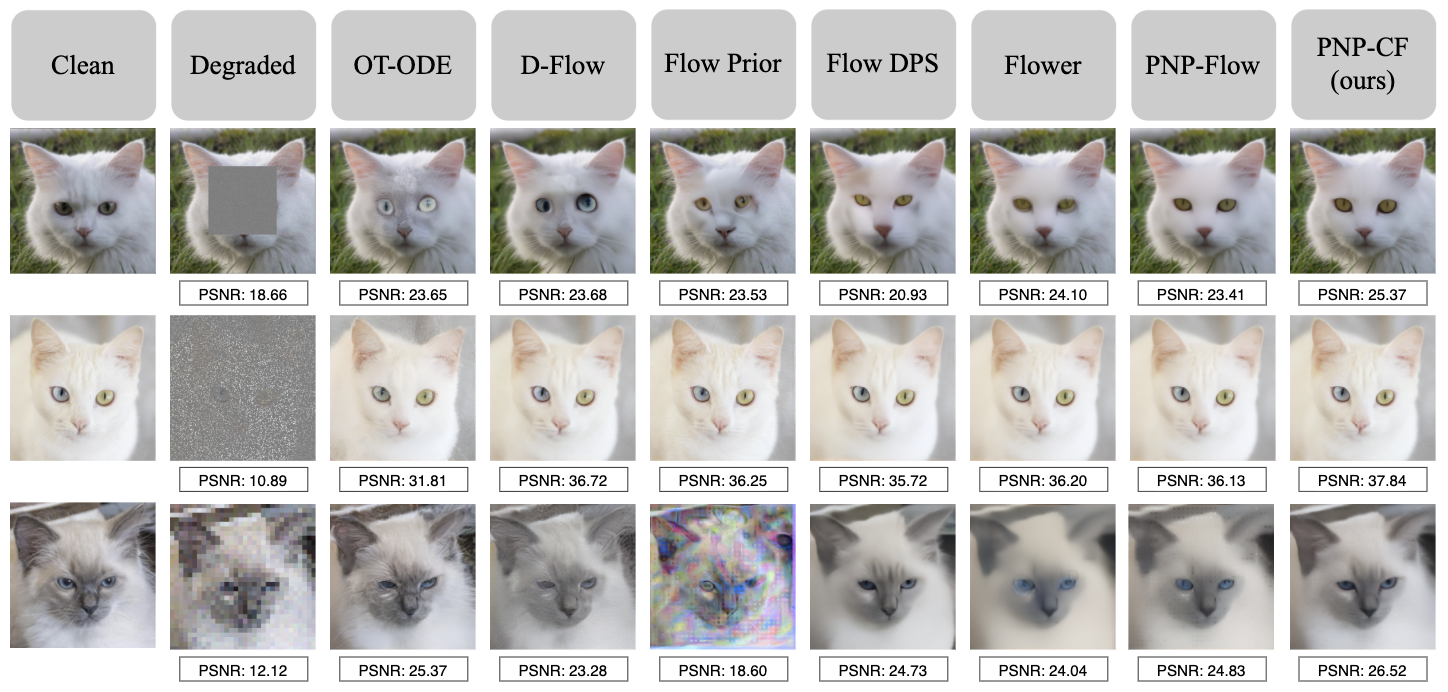}
    \caption{Additional qualitative comparisons on the AFHQ-Cat dataset across linear inverse tasks, including 
    box inpainting (1st row), random inpainting (2nd row) and super-resolution (3rd row).
    }
\end{figure}

\clearpage
\section{Additional results on nonlinear tasks with Gaussian noise}
\label{appendix:nonlinear}

\begin{figure}[H]
    \centering
    \includegraphics[width=0.9\textwidth,height=\textheight,keepaspectratio]{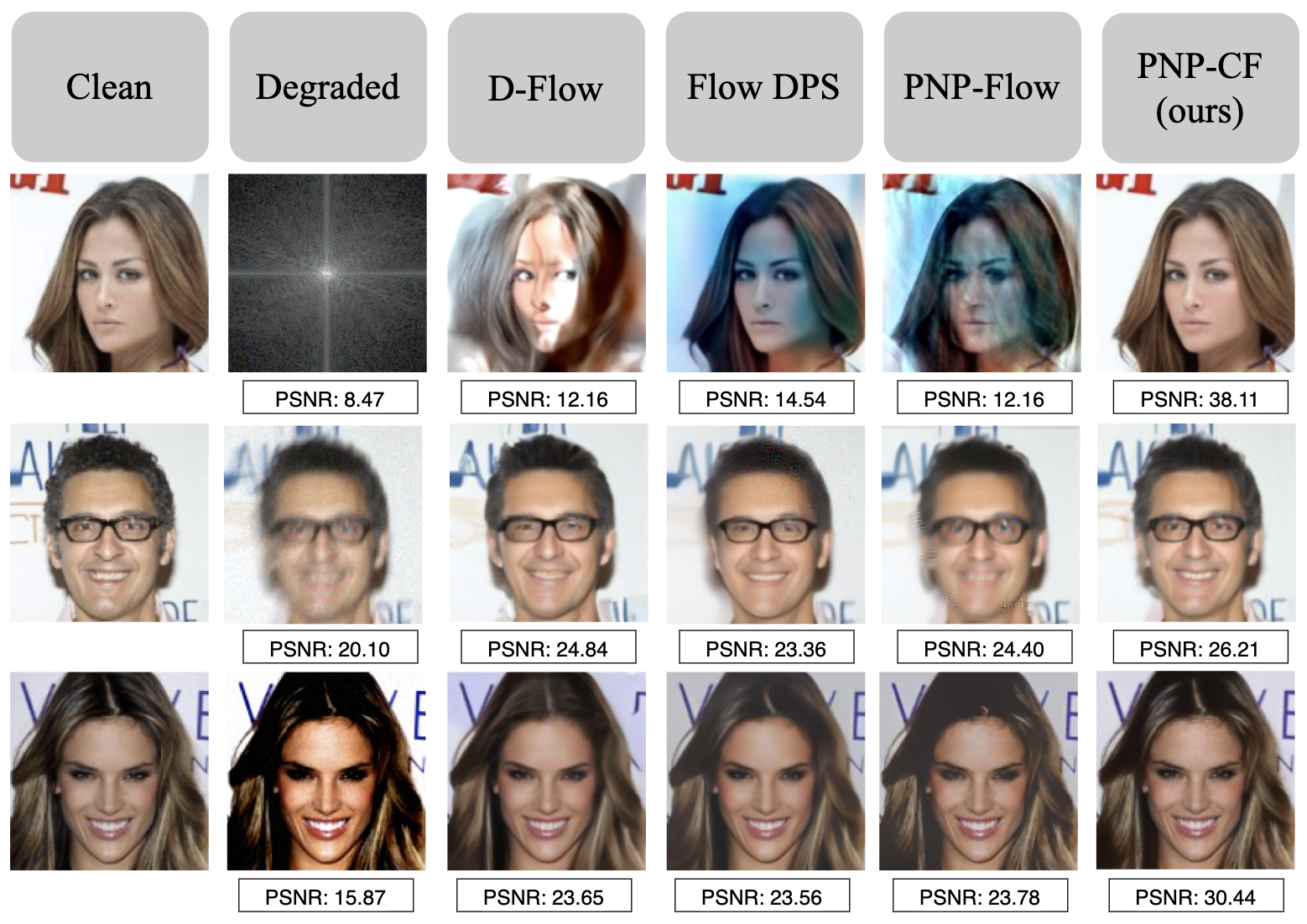}
    \caption{Additional qualitative comparisons on the CelebA dataset across nonlinear inverse tasks, including 
    phase retrieval (1st row), nonlinear deblurring (2nd row) and HDR reconstruction (3rd row).
    }
\end{figure}

\begin{figure}[H]
    \centering
    \includegraphics[width=0.9\textwidth,height=\textheight,keepaspectratio]{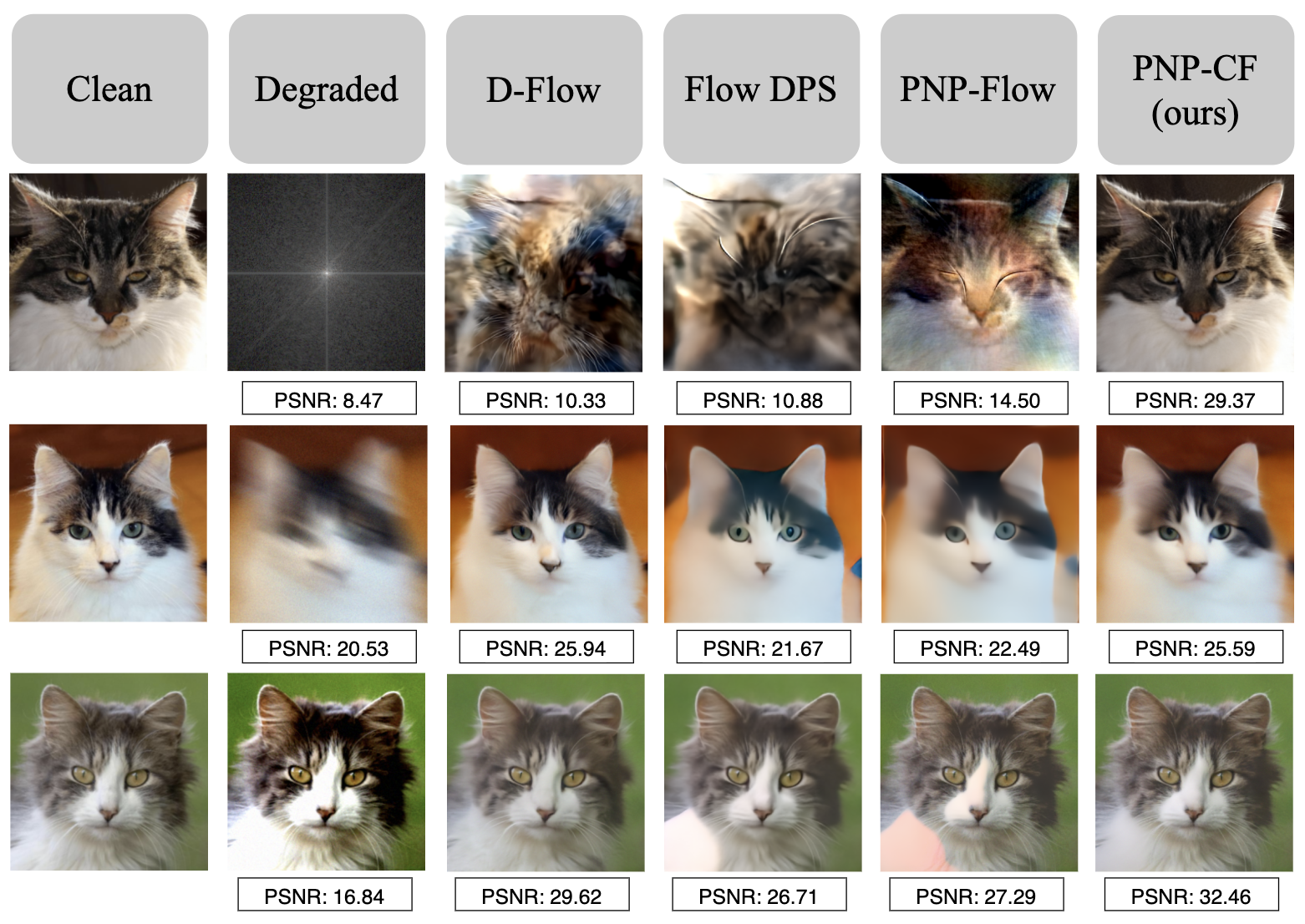}
    \caption{Additional qualitative comparisons on the AFHQ-Cat dataset across nonlinear inverse tasks, including 
    phase retrieval (1st row), nonlinear deblurring (2nd row) and HDR reconstruction (3rd row).}
\end{figure}

\clearpage
\section{Additional results on segmentation forward operator}
\label{appendix:segmentation_additional}

\begin{figure}[h]
    \centering
    \includegraphics[width=0.9\textwidth,height=\textheight,keepaspectratio]{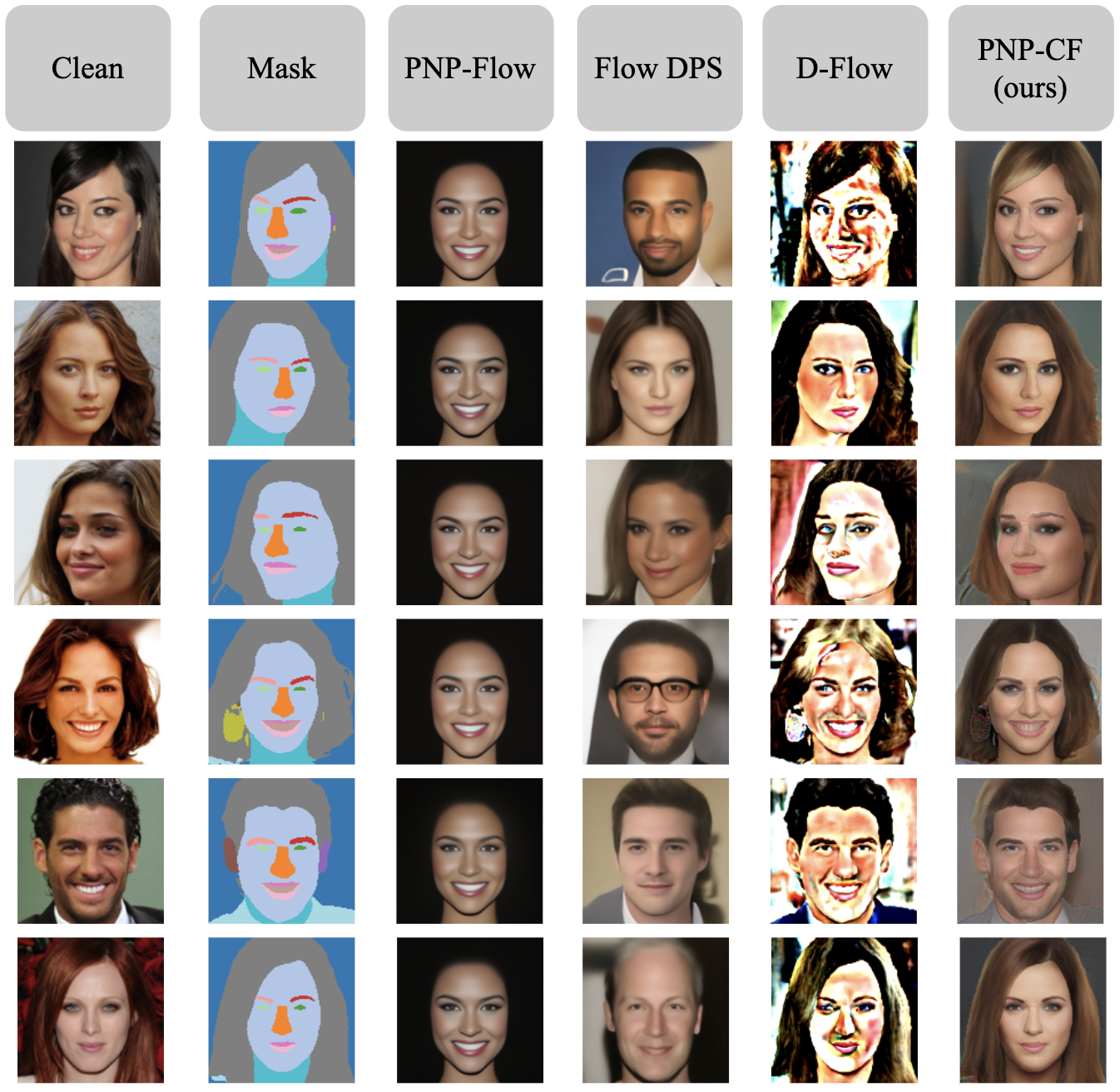}
\caption{Qualitative comparison for segmentation-guided reconstruction on CelebA-HQ.
From left to right: reference image, target segmentation mask, and reconstructions from PnP-Flow, FlowDPS, D-Flow, and PnP-CF.
PnP-Flow exhibits mode collapse across different masks, while FlowDPS produces plausible but weakly conditioned reconstructions.
D-Flow follows the segmentation constraint more closely but introduces severe visual artifacts.
In contrast, PnP-CF produces visually plausible reconstructions while maintaining strong consistency with the target segmentation.}
\end{figure}

%% file: appendix_conjugate_gradient.tex
\section{Forward-Consistency Solver}
\label{appendix:forward_models}

\subsection{The $\rvx$-Update under Additive Gaussian Noise}

Under additive Gaussian noise with variance $\sigma^2$, the forward-consistency term takes the form of a scaled squared residual,
\begin{equation}
    l(\rvy;\rvx) = \frac{1}{2\sigma^2}\|\mathcal{A}(\rvx) - \rvy\|_2^2.
\end{equation}
Inserting this expression into the $\rvx$-update of \eqref{eq:admm_updates} leads to
\begin{equation}
    \hat{\rvx}^{k+1}
    = \arg\min_{\rvx} \;
    \frac{1}{2\sigma^2}\|\mathcal{A}(\rvx) - \rvy\|_2^2
    + \frac{\rho_k}{2}\|\rvx - \tilde{\rvx}^k\|_2^2.
    \label{eq:x_subproblem}
\end{equation}
The objective balances two terms: the data-fidelity term penalizes deviations of $\mathcal{A}(\rvx)$ from the measurements, while the proximal term keeps $\rvx$ close to the current ADMM anchor $\tilde{\rvx}^k$, with a strength set by the scheduled penalty $\rho_k$ of the main algorithm.
How \eqref{eq:x_subproblem} is solved depends on the forward operator.
For linear $\mathcal{A}$, the problem is a strongly convex quadratic, which we solve with conjugate gradients (CG; appendix \ref{app:cg}).
For nonlinear $\mathcal{A}$, no closed-form minimizer is available, and we instead run a bounded number of gradient-descent iterations in which the forward-consistency gradient is obtained through a vector-Jacobian product (appendix \ref{app:nonlinear_gd}).

\subsection{Linear Forward Operators: Conjugate Gradient}
\label{app:cg}

\subsubsection{Normal Equations}
For a linear operator $\mathcal{A}$, the objective in \eqref{eq:x_subproblem} is a strongly convex quadratic, so its minimizer is unique and is characterized by a vanishing gradient,
\begin{equation}
\nabla_{\rvx}\left[
\frac{1}{2\sigma^2}\|\mathcal{A}\rvx - \rvy\|_2^2
+ \frac{\rho_k}{2}\|\rvx - \tilde{\rvx}^k\|_2^2
\right] = \mathbf{0}.
\end{equation}
Expanding the gradient and rearranging the terms gives the linear system
\begin{equation}
\underbrace{\left(
    \frac{1}{\sigma^2}\mathcal{A}^\top \mathcal{A} + \rho_k \mI
\right)}_{\mathcal{M}}
\rvx
=
\underbrace{\frac{1}{\sigma^2}\mathcal{A}^\top \rvy
+ \rho_k\tilde{\rvx}^k}_{\rvb},
\label{eq:linear_system_cg}
\end{equation}
which we write compactly as $\mathcal{M}\rvx = \rvb$, with
\begin{equation}
    \mathcal{M}(\rvx)
    = \frac{1}{\sigma^2}\mathcal{A}^\top(\mathcal{A}\rvx) + \rho_k\,\rvx,
    \qquad
    \rvb
    = \frac{1}{\sigma^2}\mathcal{A}^\top \rvy + \rho_k\tilde{\rvx}^k.
\end{equation}
The operator $\mathcal{M}$ is symmetric, since both $\mathcal{A}^\top\mathcal{A}$ and $\rho_k\mI$ are self-adjoint. It is also positive definite, since for every $\rvx \neq \mathbf{0}$
\begin{equation}
    \langle \rvx,\, \mathcal{M}(\rvx) \rangle
    = \frac{1}{\sigma^2}\|\mathcal{A}\rvx\|_2^2 + \rho_k\|\rvx\|_2^2 > 0.
\end{equation}
Consequently, \eqref{eq:linear_system_cg} has a unique solution and can be solved with CG.

\subsubsection{CG Implementation}
\label{parameters_detail}

At the image resolutions considered here, the system matrix is too large to be formed or inverted explicitly.
CG \citep{hestenes1952methods,boyd2004convex} avoids this, since it accesses $\mathcal{M}$ only through matrix-vector products, each costing one application of $\mathcal{A}$ and one of its adjoint $\mathcal{A}^\top$.
We warm-start CG from the previous iterate $\hat{\rvx}^k$ and stop once the relative residual satisfies
\begin{equation}
    \frac{\|\rvr^{(j)}\|_2}{\|\rvb\|_2 + \delta} < \texttt{cg\_tol}.
\end{equation}
Here, $\rvr^{(j)} = \rvb - \mathcal{M}\rvx^{(j)}$ is the residual after $j$ CG iterations, $\rvb$ is the right-hand side of \eqref{eq:linear_system_cg}, and the small constant $\delta > 0$ prevents division by zero.
If the tolerance is not reached, CG stops after at most \texttt{max\_cgiter} iterations.
Throughout our experiments, we set $\texttt{max\_cgiter}=10$ and $\texttt{cg\_tol}=10^{-5}$.
The base penalty $\rho$ and the schedule exponent $\beta$, in contrast, are task-dependent (table \ref{tab:hyperparams}), and each CG solve at iteration $k$ uses the resulting scheduled value $\rho_k$.

\subsection{Nonlinear Forward Operators: Gradient Descent with Backtracking}
\label{app:nonlinear_gd}

For a nonlinear $\mathcal{A}$, the subproblem \eqref{eq:x_subproblem} no longer has a closed-form minimizer.
We therefore approximate the minimizer by a fixed number of gradient-descent iterations on
\begin{equation}
    f(\rvx) = \frac{1}{2\sigma^2}\|\mathcal{A}(\rvx) - \rvy\|_2^2
    + \frac{\rho_k}{2}\|\rvx - \tilde{\rvx}^k\|_2^2,
    \label{eq:nonlinear_obj}
\end{equation}
whose gradient is
\begin{equation}
    \nabla f(\rvx) = \frac{1}{\sigma^2}\,
    \mathrm{vjp}_{\mathcal{A}}\bigl(\rvx,\, \mathcal{A}(\rvx)-\rvy\bigr)
    + \rho_k\bigl(\rvx - \tilde{\rvx}^k\bigr).
    \label{eq:nonlinear_grad}
\end{equation}
The term $\mathrm{vjp}_{\mathcal{A}}(\rvx, \rvv) = \bigl(\partial \mathcal{A}(\rvx)/\partial \rvx\bigr)^\top \rvv$ denotes the vector-Jacobian product of $\mathcal{A}$ at $\rvx$ applied to the residual $\rvv$, and is computed by automatic differentiation.
It takes the place of the explicit adjoint $\mathcal{A}^\top$ used in the linear case, without ever forming the full Jacobian.

A fixed step size is not reliable across all $(\rho_k, \sigma)$ combinations: the curvature of the data-fidelity term in \eqref{eq:nonlinear_obj} scales as $1/\sigma^2$ and, depending on the task, may exceed the curvature $\rho_k$ of the proximal term.
We therefore choose the step size by backtracking line search under the Armijo sufficient-decrease condition \citep{armijo1966minimization}.
Starting from $\tau = 1/\rho_k$, each gradient step proposes
\begin{equation}
    \rvx_{\text{try}} = \rvx - \tau\,\nabla f(\rvx),
\end{equation}
which is accepted as soon as
\begin{equation}
    f(\rvx_{\text{try}}) \le f(\rvx) - \tfrac{1}{2}\tau\,\|\nabla f(\rvx)\|_2^2;
\end{equation}
otherwise, $\tau$ is halved, for at most 20 halvings.
The next of the $\texttt{admm\_gd\_steps}$ gradient steps is warm-started from the accepted step size, doubled and capped at $1/\rho_k$, which lets later steps recover toward the initial upper bound after an overly conservative early step.
We use $\texttt{admm\_gd\_steps}=10$ for HDR reconstruction and phase retrieval and $\texttt{admm\_gd\_steps}=5$ for nonlinear deblurring, identically on both datasets.

%% file: appendix_experiments.tex
\section{Experimental Details}
\label{settings}

The overall experimental setup is described in section \ref{experiments}.
In this section, we provide additional details on the implementation of the forward operators used in our experiments, covering linear, nonlinear, and segmentation-based operators. 
We also describe how PnP-Flow and FlowDPS are adapted to the nonlinear forward models considered in our evaluation.
Finally, we report the implementation details and hyperparameter settings for PnP-CF and all baseline methods to facilitate reproducibility.

\subsection{Forward Operators with Gaussian Noise}
\label{app:forwardops}

For box and random inpainting, $\mathcal{A}$ is implemented as element-wise multiplication with the corresponding binary mask. 
For super-resolution, $\mathcal{A}$ is implemented as spatial downsampling, while for nonlinear deblurring, the forward operator is implemented using a frozen pretrained KernelWizard network with the GOPRO-wVAE checkpoint. 
Since KernelWizard operates on images in $[0,1]$, inputs are mapped from the normalized range $[-1,1]$ to $[0,1]$ before the forward pass and mapped back afterward. 
The network parameters remain fixed during reconstruction, while gradients are propagated only with respect to the input image. 
The blur condition is represented by a fixed latent kernel code shared across all evaluated images.
For phase retrieval, the input image is symmetrically zero-padded before
applying an orthonormal two-dimensional Fourier transform, after which only
the magnitude of the complex Fourier coefficients is retained. 
Since the phase information is discarded, the resulting mapping is nonlinear and highly
non-injective.
For HDR reconstruction, the forward operator first scales the image intensity and subsequently clips the result to the valid normalized range $[-1,1]$,
thereby simulating saturation and loss of information in over-exposed regions.

\subsection{Segmentation-guided Reconstruction}
\label{app:segmentation}

We evaluate segmentation-guided reconstruction on $100$ test images from CelebA-HQ \citep{karras2018progressive}, resized to $128\times128$, using $C=19$ classes. 
We use a frozen U-Net-based \citep{ronneberger2015u} semantic segmentation network trained on CelebAMask-HQ~\citep{lee2020maskgan} as the nonlinear forward operator, and a flow-matching prior originally trained on CelebA and fine-tuned on CelebA-HQ. The fine-tuning set consists of $30$k CelebA-HQ samples, split into $28.5$k images used for fine-tuning and $1.5$k held out for testing; the $100$ evaluation images are drawn from this held-out test split. 
Unlike the other inverse problems, the observation is categorical and there is no Gaussian measurement noise. 
For each reference image $\rvx^\star$, we generate the target segmentation mask by applying the forward operator and taking the most likely class at every pixel $\rvy_i = \arg\max_{c}\mathcal{A}(\rvx^\star)_{i,c}$ where $i$ indexes pixels.

Directly enforcing the one-hot segmentation mask throughout the trajectory can over-constrain early iterates, which may still be far from a plausible image. We therefore construct a softened target distribution by assigning a small amount of probability mass to a predefined set of semantically neighboring classes. Let $\rvy_{i,c}$ denote the one-hot encoding of the observed class at pixel $i$, and let $q(c\mid\rvy_i)$ denote the uniform distribution over the predefined neighbors of class $\rvy_i$. At iteration $k$, the target distribution is
\begin{equation}
\tilde{\rvy}^{k}_{i,c} = (1-\alpha_k)\rvy_{i,c} + \alpha_k q(c\mid\rvy_i),\qquad
\alpha_k=0.2\left(1-\frac{k}{K-1}\right)^3 .
\end{equation}
Thus, the target is moderately softened ($\alpha_0=0.2$) at the beginning of the trajectory and progressively approaches the original one-hot mask, with $\alpha_k\rightarrow0$ at the final iteration.

Using the same equation as in section \ref{segmentation}, the segmentation forward-consistency term is given as 
$l_{\mathrm{seg}}(\tilde{\rvy}^{k};\rvx)
=-\sum_i\sum_{c=1}^{C}\tilde{\rvy}^{k}_{i,c}\log \mathcal{A}(\rvx)_{i,c}$.
The corresponding $\rvx$-update approximately minimizes
$\lambda_{\mathrm{CE}}\,l_{\mathrm{seg}}(\tilde{\rvy}^{k};\rvx)
+\frac{\rho_{k}}{2}\|\rvx-(\rvz^k-\rvu^k)\|_2^2$.
Here, $\lambda_{\mathrm{CE}}$ controls the relative contribution of the cross-entropy forward-consistency term with respect to the ADMM quadratic penalty, and is set to $\lambda_{\mathrm{CE}}=10$ for all segmentation experiments.
The remaining optimization and schedule hyperparameters follow the same formulation as for the other inverse problems, with the task-specific values reported in table \ref{tab:hyperparams}.

\begin{table}[t]
\centering

\caption{Hyperparameter settings per task and dataset for the time-adaptive
PnP-CF variant. $K$: number of PnP-CF iterations (\texttt{steps\_pnp}); $\rho$:
base ADMM penalty (\texttt{rho\_admm}); $\beta$: penalty-schedule exponent
(\texttt{beta\_rho}); $\gamma$: time-schedule exponent
(\texttt{t\_schedule\_gamma}); Denoiser calls: number of denoiser samples
averaged per iteration (\texttt{num\_denoiser\_calls}); GD steps: nonlinear
x-update gradient steps (\texttt{admm\_gd\_steps}, linear tasks use CG instead).}

\label{tab:hyperparams}
\small
\setlength{\tabcolsep}{3pt}
\begin{tabular}{llcccccc}
\toprule
Dataset & Task & $K$ & $\rho$ & $\beta$ & $\gamma$ & Denoiser calls & GD steps \\
\midrule
CelebA & Box inpainting       & 100 & 50   & 2.5 & 1.0 & 5 & -- \\
CelebA & Random inpainting    & 100 & 100  & 2.5 & 1.0 & 5 & -- \\
CelebA & Super-resolution     & 100 & 50   & 1.3 & 1.0 & 5 & -- \\
CelebA & Phase retrieval      & 100 & 1500 & 1.8 & 1.5 & 5 & 10 \\
CelebA & HDR reconstruction   & 100 & 100  & 1.5 & 0.2 & 5 & 10 \\
CelebA & Nonlinear deblurring & 200 & 100  & 1.5 & 1.5 & 1 & 5 \\
\midrule
AFHQ-Cat & Box inpainting       & 100 & 50   & 2.5 & 1.5 & 5 & -- \\
AFHQ-Cat & Random inpainting    & 100 & 100  & 2.5 & 1.0 & 5 & -- \\
AFHQ-Cat & Super-resolution     & 500 & 50   & 1.3 & 1.0 & 5 & -- \\
AFHQ-Cat & Phase retrieval      & 300 & 1500 & 2.2 & 1.9 & 5 & 10 \\
AFHQ-Cat & HDR reconstruction   & 100 & 500  & 1.5 & 1.0 & 5 & 10 \\
AFHQ-Cat & Nonlinear deblurring & 250 & 100  & 1.5 & 1.5 & 1 & 5 \\
\midrule
CelebA-HQ & Segmentation guided & 500 & 500  & 0 & 1.0 & 1 & 10 \\
\bottomrule
\end{tabular}
\end{table}

\subsection{Hyperparameters}
\label{app:hyperparams_intro}

The main hyperparameter in our method is the penalty parameter $\rho$, and to further improve reconstruction performance, we consider the schedules $\rho_k=\rho/(1-t_k)^\beta$ and $t_k=(k/K)^\gamma$, where $\beta$ controls the penalty growth as $t_k\to1$, while $\gamma$ controls the progression along the flow trajectory.

For nonlinear deblurring, we observed a small improvement in reconstruction quality from an optional early-stopping at $85\%$ of the configured iteration budget $K$, corresponding to $k=0.85K$. This choice is not critical to the method and yields only a minor gain; all other tasks run for the full $K$ iterations.

For phase retrieval, two random restarts are used to reduce sensitivity to initialization, following the restart strategy commonly adopted for this task in DPS-based inverse-problem settings \citep{chung2024diffusion}. 
This restart procedure is applied consistently to all evaluated methods, and for each image the reconstruction with the lowest final measurement residual $\|\mathcal{A}(\rvx)-\rvy\|_2$ is retained.

Table \ref{tab:hyperparams} reports the full hyperparameter configuration used for each task and dataset, corresponding to the time-adaptive variant of PnP-CF reported in table \ref{tab:results}. 
All linear tasks share $\texttt{max\_cgiter}=10$ and $\texttt{cg\_tol}=10^{-5}$ (\ref{app:cg}); GD steps apply only to the nonlinear tasks (\ref{app:nonlinear_gd}).

\subsubsection{\texorpdfstring{Role of $\gamma$ and $\beta$}{Role of gamma and beta}}
\label{app:role_hyperparams}
The schedule parameters $\gamma$ and $\beta$ control two distinct aspects of the
algorithm and are therefore treated as separate hyperparameters. 
The parameter $\gamma$ determines how the iteration index is mapped to the
flow time, $t_k = \left(\frac{k}{K}\right)^{\gamma},$ and thus controls the rate at which the iterates progress along the learned flow trajectory: $\gamma = 1$ gives a linear progression, $\gamma > 1$ keeps the iterates at earlier flow times longer before accelerating near the end, and $\gamma < 1$ advances more rapidly toward later flow times — with $\gamma < 1$ correspondingly favoring the deterministic component $\rvepsilon_d$ of the pullback step earlier in the trajectory, and $\gamma > 1$ retaining the stochastic component $\rvepsilon_s$ for longer.

In contrast, $\beta$ acts only after $t_k$ has been determined, through the ADMM penalty schedule
\[
\rho_k = \frac{\rho}{(1-t_k)^{\beta}} = \frac{\rho}{\left(1-(k/K)^{\gamma}\right)^{\beta}},
\]
and controls how aggressively the consensus penalty grows as the algorithm
approaches the end of the trajectory.

The distinction between these roles can be seen from the terminal growth rate and the onset of penalty growth.
As $k \to K$, writing $\varepsilon = 1-k/K \to 0^{+}$, we have
$t_k = (1-\varepsilon)^{\gamma} \approx 1-\gamma\varepsilon$ and hence
\[
\rho_k \;\approx\; \rho\,\gamma^{-\beta}\,\varepsilon^{-\beta}, \qquad k\to K,
\]
so, for the unclipped schedule, the asymptotic order of the terminal penalty
growth is determined by $\beta$, while $\gamma$ enters only through the
multiplicative factor $\gamma^{-\beta}$ near $k=K$. 
Meanwhile, for fixed $\beta$, the onset of penalty growth, defined here as the
iteration fraction $s^{*}=k^{*}/K$ at which $\rho_{k^{*}}$ has doubled relative
to $\rho$, is
\[
s^{*} = \left(1-2^{-1/\beta}\right)^{1/\gamma},
\]
which is governed by $\gamma$ for any fixed $\beta$. 
Consequently, a single schedule parameter cannot
independently control both properties: varying $\beta$ changes the terminal
growth rate but also shifts its onset, whereas varying $\gamma$ can shift the
onset without changing the asymptotic growth order. 
Adjusting both parameters therefore provides separate degrees of freedom for
controlling the timing and strength of the penalty growth.
For example, a penalty that remains moderate for most iterations
but grows aggressively near the end can be obtained by jointly controlling
the delayed onset through $\gamma$ and the terminal growth rate through
$\beta$.

Beyond this functional distinction, the two parameters are also mathematically independent. 
Suppose two parameter pairs
$(\beta_1,\gamma_1)$ and $(\beta_2,\gamma_2)$ induced the same penalty
schedule at every iteration, i.e.
\[
\left(1-s^{\gamma_1}\right)^{-\beta_1} = \left(1-s^{\gamma_2}\right)^{-\beta_2}
\qquad \forall\, s\in(0,1).
\]

\begin{proposition}
If the above identity holds for all $s\in(0,1)$ with
$\beta_1,\beta_2,\gamma_1,\gamma_2 > 0$, then $\beta_1=\beta_2$ and
$\gamma_1=\gamma_2$.
\end{proposition}

\begin{proof}
Taking logarithms, the identity is equivalent to
\[
\frac{\beta_1}{\beta_2} = \frac{\ln\!\left(1-s^{\gamma_2}\right)}{\ln\!\left(1-s^{\gamma_1}\right)}
=: R(s), \qquad \forall\, s \in (0,1).
\]
Since the left-hand side is a fixed constant, $R(s)$ must be constant on
$(0,1)$. Using $\ln(1-x) = -x + O(x^2)$ as $x\to 0^{+}$,
\[
R(s) = \frac{s^{\gamma_2} + O\!\left(s^{2\gamma_2}\right)}{s^{\gamma_1} + O\!\left(s^{2\gamma_1}\right)}
= s^{\gamma_2-\gamma_1}\bigl(1+o(1)\bigr), \qquad s \to 0^{+}.
\]
If $\gamma_1 \neq \gamma_2$, this expression tends to either $0$ or $\infty$
as $s \to 0^{+}$, contradicting that $R$ is constant. Hence
$\gamma_1 = \gamma_2 =: \gamma$. Substituting back, the original identity
reduces to $\beta_1 \ln(1-s^{\gamma}) = \beta_2 \ln(1-s^{\gamma})$ for all
$s \in (0,1)$; since $\ln(1-s^{\gamma}) \neq 0$ on $(0,1)$, this forces
$\beta_1 = \beta_2$.
\end{proof}

Hence no reparametrization of $\gamma$ can absorb a change in $\beta$, or vice
versa: the two parameters are distinct, non-redundant degrees of freedom.
The parameter $\gamma$ governs the temporal progression along the flow
trajectory and influences the onset of penalty growth, whereas $\beta$
determines the strength and asymptotic growth rate of the ADMM consensus
penalty.

\subsection{Adapting baselines to Nonlinear Operators}
\label{app:baseline_adaptation}

PnP-Flow \citep{martin2025pnpflow} and FlowDPS
\citep{kim2025flowdps} were originally evaluated only on linear inverse
problems.
For PnP-Flow, the linear forward-consistency gradient involves the adjoint
$\mathcal{A}^\top$. We extend this update to nonlinear operators by replacing
the adjoint-based gradient with the corresponding vector-Jacobian product
(VJP), $\mathrm{vjp}_{\mathcal{A}}(\rvx,\rvv) =
\left(\frac{\partial\mathcal{A}(\rvx)}{\partial\rvx}\right)^\top \rvv,$
computed via automatic differentiation. For linear $\mathcal{A}$, the VJP
reduces exactly to $\mathcal{A}^\top\rvv$; for nonlinear $\mathcal{A}$, it
provides the required forward-consistency gradient without requiring a closed-form
adjoint. This is the same mechanism used in PnP-CF's nonlinear
forward-consistency step (\ref{app:nonlinear_gd}).

For FlowDPS, we follow the practical strategy of the authors and perform a
fixed number of gradient-descent updates on the forward-consistency objective.
For nonlinear operators, we compute the gradient of
$\|\rvy-\mathcal{A}(\rvz)\|_2^2$ using the VJP of $\mathcal{A}$.
This preserves the key FlowDPS approximation that avoids differentiating
through the flow-matching denoiser network.

D-Flow \citep{benhamu2024dflow} requires no structural adaptation to handle nonlinear operators: it optimizes the initial noise directly by backpropagating through the full generation trajectory and the forward operator via standard automatic differentiation, which places no restriction on the linearity of $\mathcal{A}$. 
However, D-Flow's own published experiments on image inverse problems cover only linear
degradations (inpainting, super-resolution, Gaussian deblurring); its reported validation on nonlinear problems is limited to non-image domains (audio, molecule generation). 
We therefore run D-Flow on our nonlinear image tasks using its unmodified optimization procedure.

Tables~\ref{tab:hyperparams_baselines_celeba} and \ref{tab:hyperparams_baselines_afhq} report the full per-task hyperparameter settings used for all baselines on CelebA and AFHQ-Cat, respectively. 
For the nonlinear tasks, we tuned PnP-Flow, FlowDPS, and D-Flow to obtain the best reconstruction quality achievable with each method, while keeping their iteration budgets and computational cost comparable to PnP-CF whenever possible.

\begin{table}[h]
\centering
\caption{Hyperparameters used for all baseline methods on the CelebA
dataset.}
\label{tab:hyperparams_baselines_celeba}
\scriptsize
\begin{tabular}{l cccccc}
\toprule
& Box Inpaint. & Super-res. & Rand. Inpaint. & HDR Recon. & Phase Retr. & Nonl. Deblur. \\
\midrule
\rowcolor{gray!20} OT-ODE & & & & & & \\
$t_0$ (start time)     & 0.1 & 0.1 & 0.1 & N/A & N/A & N/A \\
$\gamma$ schedule      & $\gamma_t$ & constant & constant & N/A & N/A & N/A \\
$K$ (steps)            & 100 & 100 & 100 & N/A & N/A & N/A \\
\midrule
\rowcolor{gray!20} Flow Priors & & & & & & \\
$\lambda$ (regularization) & 10{,}000 & 10{,}000 & 10{,}000 & N/A & N/A & N/A \\
$\eta$ (learning rate)     & 0.01 & 0.1 & 0.01 & N/A & N/A & N/A \\
\midrule
\rowcolor{gray!20} D-Flow & & & & & & \\
$\lambda$ (regularization) & 0.001 & 0.001 & 0.01 & 0.01 & 0.01 & 0.01 \\
$\alpha$ (blending)        & 0.1   & 0.1   & 0.1  & 0.1  & 0.1  & 0.1  \\
$n_{\mathrm{iter}}$ (max.\ iter.) & 9 & 10 & 20 & 10 & 10 & 10 \\
\midrule
\rowcolor{gray!20} FlowDPS & & & & & & \\
Grad.\ step size & 2.0 & 2.5 & 3.0 & 0.50 & 2.0 & 0.5 \\
Grad.\ steps      & 3   & 3   & 3   & 3    & 3   & 3   \\
$K$ (steps)       & 100 & 100 & 100 & 100  & 100 & 200 \\
\midrule
\rowcolor{gray!20} FLOWER & & & & & & \\
Denoiser calls & 5   & 5   & 5   & N/A & N/A & N/A \\
$K$ (steps)   & 100 & 100 & 100 & N/A & N/A & N/A \\
\midrule
\rowcolor{gray!20} PnP-Flow & & & & & & \\
$\alpha$ (learning rate factor) & 0.5 & 0.3 & 0.01 & 0.3 & 0.01 & 0.5 \\
Denoiser calls                  & 5   & 5   & 5    & 5   & 5    & 5   \\
$K$ (steps)                     & 100 & 100 & 100  & 100 & 100  & 200 \\
\bottomrule
\end{tabular}
\end{table}

\begin{table}[t]
\centering
\caption{Hyperparameters used for all baseline methods on the AFHQ-Cat
dataset.}
\label{tab:hyperparams_baselines_afhq}
\scriptsize
\begin{tabular}{l cccccc}
\toprule
& Box Inpaint. & Super-res. & Rand. Inpaint. & HDR Recon. & Phase Retr. & Nonl. Deblur.\\
\midrule
\rowcolor{gray!20} OT-ODE & & & & & & \\
$t_0$ (start time)     & 0.1 & 0.1 & 0.1 & N/A & N/A & N/A \\
$\gamma$ schedule      & $\gamma_t$ & constant & constant & N/A & N/A & N/A \\
$K$ (steps)            & 100 & 500 & 100 & N/A & N/A & N/A \\
\midrule
\rowcolor{gray!20} Flow Priors & & & & & & \\
$\lambda$ (regularization) & 10{,}000 & 10{,}000 & 10{,}000 & N/A & N/A & N/A \\
$\eta$ (learning rate)     & 0.01 & 0.1 & 0.01 & N/A & N/A & N/A \\
\midrule
\rowcolor{gray!20} D-Flow & & & & & & \\
$\lambda$ (regularization) & 0.01  & 0.001 & 0.001 & 0.01 & 0.01 & 0.01 \\
$\alpha$ (blending)        & 0.1   & 0.1   & 0.1   & 0.1  & 0.1  & 0.1  \\
$n_{\mathrm{iter}}$ (max.\ iter.) & 9 & 20 & 20 & 10 & 10 & 10 \\
\midrule
\rowcolor{gray!20} FlowDPS & & & & & & \\
Grad.\ step size & 2.0 & 2.5 & 3.0 & 0.5 & 0.5 & 0.10 \\
Grad.\ steps      & 3   & 3   & 3   & 3   & 3   & 3    \\
$K$ (steps)       & 100 & 500 & 100 & 100 & 300 & 250  \\
\midrule
\rowcolor{gray!20} FLOWER & & & & & & \\
Denoiser calls & 5   & 5   & 5   & N/A & N/A & N/A \\
$K$ (steps)   & 100 & 500 & 100 & N/A & N/A & N/A \\
\midrule
\rowcolor{gray!20} PnP-Flow & & & & & & \\
$\alpha$ (learning rate factor) & 0.5 & 0.01 & 0.01 & 0.3 & 0.01 & 0.8 \\
Denoiser calls                  & 5   & 5    & 5    & 5   & 5    & 5   \\
$K$ (steps)                     & 100 & 500  & 100  & 100 & 300  & 250 \\
\bottomrule
\end{tabular}
\end{table}

\section{Computational Complexity and Efficiency}
\label{appendix:computational}

In this section, we analyze the computational efficiency and memory footprint of all methods.
Experiments are conducted on a single NVIDIA A100-SXM4-80GB GPU (CUDA 13.0), using a batch size of 4 images and averaged over 25 batches keeping the same hyperparameter settings as used in the main results (\ref{tab:results}).
Table \ref{tab:efficiency} reports per-batch runtime and peak GPU memory usage, averaged separately over the linear tasks (box inpainting, random inpainting, super-resolution) and the nonlinear tasks (HDR reconstruction, phase retrieval, nonlinear deblurring).

\begin{table*}[h]
\centering
\caption{Average per-batch runtime (in seconds) and peak GPU memory (in GB) on CelebA and AFHQ-Cat, grouped by linear tasks (box inpainting, random inpainting, super-resolution) and nonlinear tasks (HDR reconstruction, phase retrieval, nonlinear deblurring), each averaged over the three tasks in its group (batch size of 4 images, averaged over 25 batches). Best scores are \best{bold}.}
\label{tab:efficiency}
\begin{minipage}[t]{0.49\linewidth}
\centering
\setlength{\tabcolsep}{3pt}
\renewcommand{\arraystretch}{1.15}
\begin{adjustbox}{width=\linewidth}
\begin{tabular}{l cc cc}
\toprule
& \multicolumn{2}{c}{Linear} & \multicolumn{2}{c}{Nonlinear} \\
\cmidrule(lr){2-3} \cmidrule(lr){4-5}
Method & Time (s) & Memory (GB) & Time (s) & Memory (GB) \\
\midrule
PnP-CF (ours) & 16.81 & \best{0.32} & 32.62 & \best{0.46} \\
PnP-Flow      & 17.31          & \best{0.32} & 27.56          & \best{0.46} \\
OT-ODE        & 17.49          & 2.34 & --             & -- \\
D-Flow        & 87.77          & 13.36 & 92.19         & 17.84 \\
Flow Priors   & 35.66          & 8.81 & --             & -- \\
FlowDPS       & \best{10.83}   & 0.34 & \best{20.97} & 0.47 \\
FLOWER        & 19.13          & 0.33 & --             & -- \\
\bottomrule
\end{tabular}
\end{adjustbox}
\centering
\footnotesize CelebA
\end{minipage}%
\hfill
\begin{minipage}[t]{0.49\linewidth}
\centering
\setlength{\tabcolsep}{3pt}
\renewcommand{\arraystretch}{1.15}
\begin{adjustbox}{width=\linewidth}
\begin{tabular}{l cc cc}
\toprule
& \multicolumn{2}{c}{Linear} & \multicolumn{2}{c}{Nonlinear} \\
\cmidrule(lr){2-3} \cmidrule(lr){4-5}
Method & Time (s) & Memory (GB) & Time (s) & Memory (GB) \\
\midrule
PnP-CF (ours) & 50.52          & \best{0.85} & 69.11          & \best{0.99} \\
PnP-Flow      & 46.17          & \best{0.85} & 63.93 & \best{0.99} \\
OT-ODE        & 44.96 & 8.55        & --             & -- \\
D-Flow        & 141.75         & 50.47       & 113.08         & 50.80 \\
Flow Priors   & 75.97          & 33.30       & --             & -- \\
FlowDPS       & \best{35.01}   & 0.87        & \best{47.61}   & 1.01 \\
FLOWER        & 59.71          & 0.86        & --             & -- \\
\bottomrule
\end{tabular}
\end{adjustbox}
\centering
\footnotesize AFHQ-Cat
\end{minipage}
\end{table*}

D-Flow and Flow Priors incur significant computational and memory overhead, limiting practical efficiency.
D-Flow is expensive because it repeatedly solves the full ODE and backpropagates through the entire trajectory to optimize the initial noise.
This leads to high memory usage from storing intermediate activations and long runtime due to repeated full-trajectory differentiation.
Flow Priors is computationally expensive, as each ODE step involves multiple gradient updates, requiring repeated forward and backward
passes through the model along with costly Jacobian trace evaluations. This nested optimization leads to high overall runtime.
OT-ODE avoids full trajectory backpropagation by applying a forward-consistency correction at each ODE step through a vector-Jacobian product (VJP) with respect to the \emph{denoiser network}; because this network is large, computing the VJP still requires storing its intermediate activations at every step, incurring notable memory overhead despite avoiding full-trajectory differentiation.
Neither OT-ODE, Flow Priors, nor FLOWER report nonlinear extensions, so their nonlinear-task averages are omitted.

In contrast, PnP-Flow, FlowDPS, FLOWER, and PnP-CF generally incur lower computational cost and substantially lower memory usage than D-Flow and Flow Priors on the evaluated tasks.
In each case, the forward-consistency step avoids differentiating through the denoiser network.
Among the evaluated methods, FlowDPS achieves the lowest average runtime
across both linear and nonlinear tasks. 
PnP-CF solves the linear forward-consistency subproblem iteratively using
conjugate gradients (appendix \ref{app:cg}), and uses a bounded number of
gradient-descent steps with backtracking and VJPs for nonlinear operators
(appendix \ref{app:nonlinear_gd}), without backpropagating through the
denoiser in either case. 
These more accurate forward-consistency updates introduce additional runtime relative to FlowDPS, but keep PnP-CF's memory usage low and comparable to PnP-Flow and FlowDPS across both linear and nonlinear tasks, and to FLOWER on the linear tasks
(table \ref{tab:efficiency}).

%% file: ablation.tex
\section{Ablation Studies}
\label{ablation}



In this section, we further study the impact of the main design choices in PnP-CF through additional ablation experiments. 
In particular, we analyze the role of the pullback construction and the dual-variable update, which together control how the intermediate iterates balance exploration of the flow prior with consistency to the current reconstruction trajectory. 
These experiments complement the analysis presented in section \ref{pnpcf_analysis} and provide additional evidence for the contribution of each component.

First, we evaluate PnP-CF with only stochastic component ($\rvepsilon_s$). 
Second, we study PnP-CF with only deterministic component  ($\rvepsilon_d$). 
Finally, we consider our full approach, PnP-CF with time-adaptive pullback strategy, which combines both stochastic and deterministic components in an adaptive manner 
within the consensus framework.

Table \ref{tab:variant_results_by_task} reports the performance across different variants of our method.
Complementing these results, figure \ref{fig:ablation_appendix_celeba_cats} provides a qualitative comparison on CelebA and AFHQ-Cat across the six inverse problems considered in our experiments.

Figure \ref{fig:dual_analysis} presents the PSNR, SSIM and LPIPS  curves for different inverse tasks, comparing the settings with and without dual variable across optimization iterations. 

\begin{figure}[h]
    \centering
    \includegraphics[width=\linewidth]{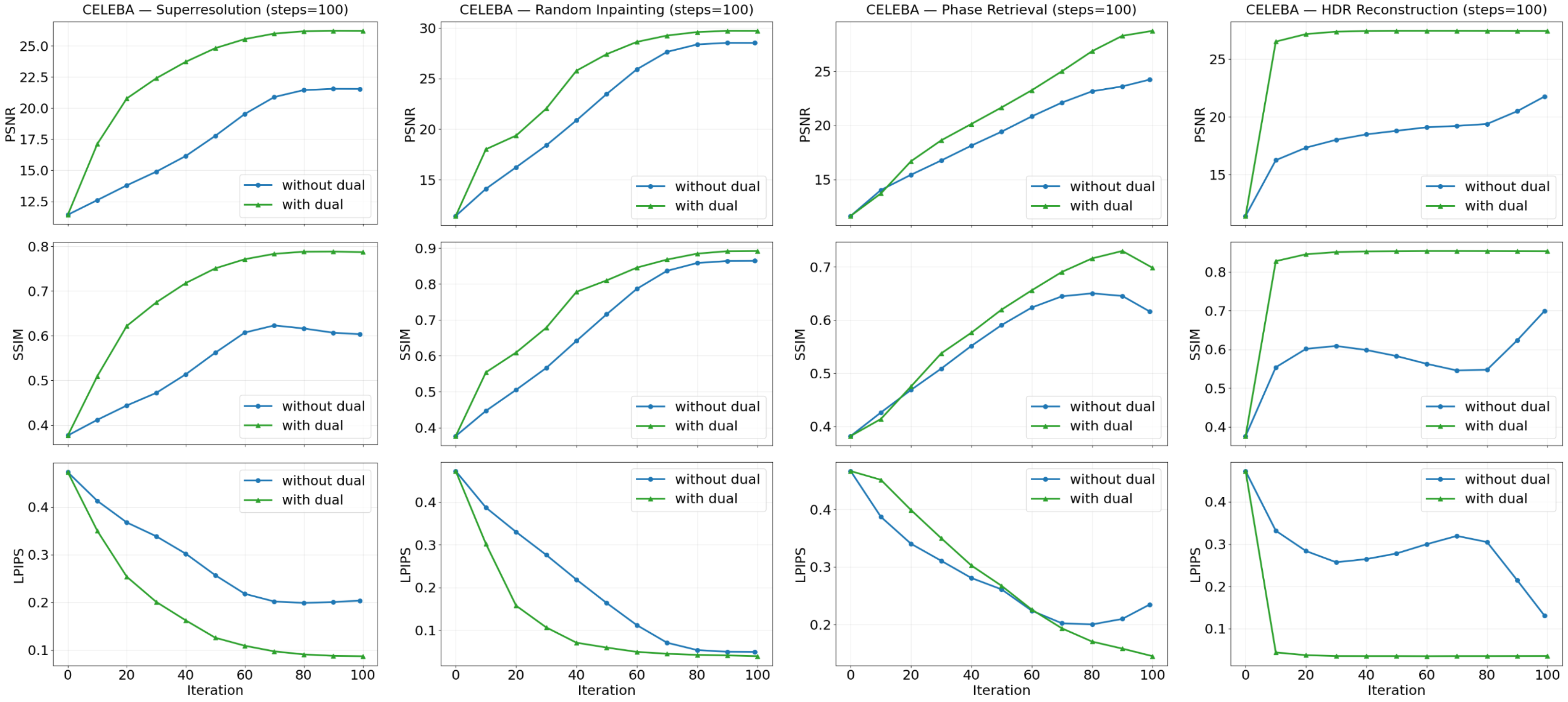}
\caption{Effect of the dual-variable update on reconstruction quality across optimization iterations. PSNR (top), SSIM (middle), and LPIPS (bottom) are shown for super-resolution, random inpainting, phase retrieval, and HDR reconstruction on CelebA, comparing PnP-CF `with' and `without' the dual-variable update. Across tasks, the dual-variable formulation generally improves convergence speed and final reconstruction quality, with particularly large gains for super-resolution, random inpainting, and HDR reconstruction.}
\label{fig:dual_analysis}
\end{figure}

\begin{table}[ht]
\centering
\caption{Task-wise performance of the three PnP-CF pullback variants
(time-adaptive, stochastic, deterministic) across all six inverse tasks. Best
result per task and metric, among the three variants, is shown in \textbf{bold}.}
\label{tab:variant_results_by_task}
\scriptsize
\resizebox{\textwidth}{!}
{%
\begin{tabular}{ll ccc ccc}
\toprule
& & \multicolumn{3}{c}{CelebA} & \multicolumn{3}{c}{AFHQ-Cat} \\
\cmidrule(lr){3-5} \cmidrule(lr){6-8}
Task & Method & PSNR$\uparrow$ & LPIPS$\downarrow$ & SSIM$\uparrow$ & PSNR$\uparrow$ & LPIPS$\downarrow$ & SSIM$\uparrow$ \\
\midrule
\multirow{3}{*}{HDR Reconstruction}
& PnP-CF (adaptive)      & \best{27.59} & \best{0.04} & \best{0.85} & \best{26.29} & \best{0.08} & \best{0.86} \\
& PnP-CF (stochastic)    & 22.44 & 0.10 & 0.80 & 22.97 & 0.32 & 0.71 \\
& PnP-CF (deterministic) & 18.12 & 0.36 & 0.43 & 24.11 & 0.13 & 0.75 \\
\midrule
\multirow{3}{*}{Box Inpainting}
& PnP-CF (adaptive)      & \best{26.36} & \best{0.05} & \best{0.89} & \best{22.65} & \best{0.13} & \best{0.82} \\
& PnP-CF (stochastic)    & 22.79 & 0.18 & 0.74 & 21.24 & 0.37 & 0.68 \\
& PnP-CF (deterministic) & 22.50 & 0.10 & 0.72 & 20.65 & 0.21 & 0.66 \\
\midrule
\multirow{3}{*}{Nonlinear Deblurring}
& PnP-CF (adaptive)      & \best{27.09} & \best{0.09} & \best{0.79} & \best{22.28} & \best{0.29} & \best{0.58} \\
& PnP-CF (stochastic)    & 25.71 & 0.15 & 0.77 & 22.27 & 0.32 & 0.55 \\
& PnP-CF (deterministic) & 13.46 & 0.68 & 0.24 & 20.01 & 0.58 & 0.35 \\
\midrule
\multirow{3}{*}{Phase Retrieval}
& PnP-CF (adaptive)      & \best{30.33} & \best{0.11} & \best{0.80} & \best{25.27} & \best{0.22} & \best{0.70} \\
& PnP-CF (stochastic)    & 23.72 & 0.21 & 0.70 & 20.22 & 0.47 & 0.57 \\
& PnP-CF (deterministic) & 23.67 & 0.26 & 0.61 & 15.82 & 0.56 & 0.34 \\
\midrule
\multirow{3}{*}{Random Inpainting}
& PnP-CF (adaptive)      & \best{29.93} & \best{0.04} & \best{0.89} & \best{30.72} & \best{0.07} & \best{0.85} \\
& PnP-CF (stochastic)    & 27.50 & 0.12 & 0.83 & 28.25 & 0.29 & 0.78 \\
& PnP-CF (deterministic) & 21.93 & 0.25 & 0.59 & 22.01 & 0.43 & 0.48 \\
\midrule
\multirow{3}{*}{Super Resolution}
& PnP-CF (adaptive)      & \best{26.45} & \best{0.08} & \best{0.79} & \best{23.72} & \best{0.27} & \best{0.66} \\
& PnP-CF (stochastic)    & 23.38 & 0.19 & 0.72 & 19.70 & 0.48 & 0.57 \\
& PnP-CF (deterministic) & 18.17 & 0.34 & 0.42 & 14.05 & 0.68 & 0.27 \\
\bottomrule
\end{tabular}}
\end{table}

\begin{figure}[h]
    \centering
    \begin{minipage}[h]{0.48\linewidth}
        \centering
        \includegraphics[
            width=\linewidth,
            height=\textheight,
            keepaspectratio
        ]{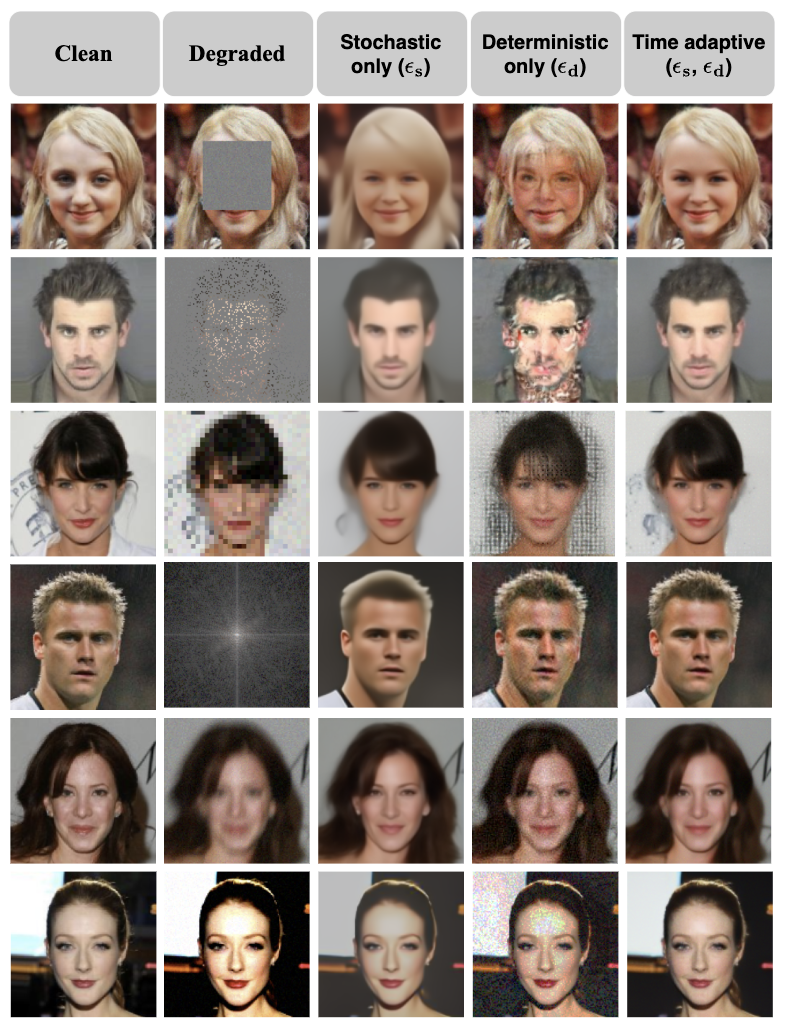}
    \end{minipage}
    \hfill
    \begin{minipage}[h]{0.48\linewidth}
        \centering
        \includegraphics[
            width=\linewidth,
            height=\textheight,
            keepaspectratio
        ]{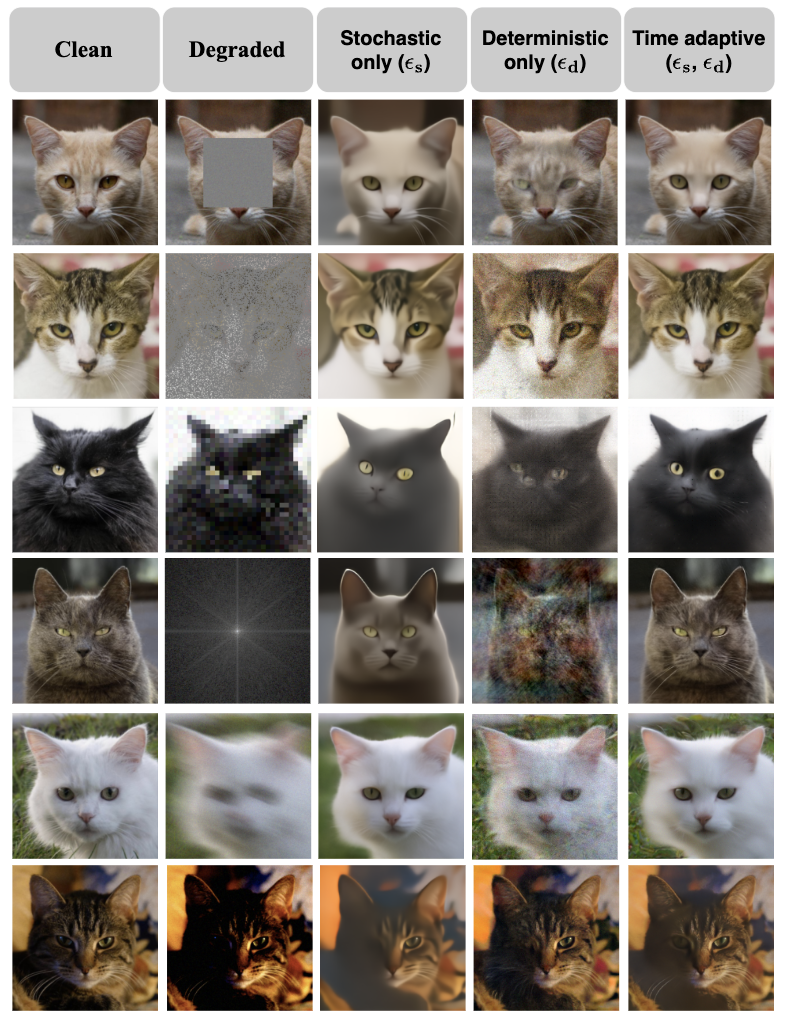}
    \end{minipage}
    \caption{Additional qualitative ablation results comparing the pullback variants of PnP-CF on CelebA (left) and AFHQ-Cat (right).
    Within each dataset, columns from left to right show the clean image, degraded observation, stochastic-only ($\rvepsilon_s$), deterministic-only ($\rvepsilon_d$), and time-adaptive ($\rvepsilon_s,\rvepsilon_d$) variants. 
    From top to bottom, the rows correspond to 
    box inpainting, random inpainting, super-resolution, phase retrieval, nonlinear deblurring, and HDR reconstruction.}
    \label{fig:ablation_appendix_celeba_cats}
\end{figure}

\section{Additional Experiments}
\label{app:additonal}

\subsection{Segmentation guidance failure cases}
\label{app:seg_failures}

Although PnP-CF performs well overall for segmentation-guided reconstruction, as shown in the results in sections \ref{segmentation} and appendix \ref{appendix:segmentation_additional}, we observe occasional failure cases that reveal an important limitation of using a learned segmentation network as the forward operator.

The reconstruction is optimized to reproduce the target segmentation under a fixed segmentator. Consequently, the optimization may exploit imperfections or ambiguities in the learned forward model and produce image structures that induce the desired segmentation output without corresponding to a perceptually plausible face. 
In such cases, high segmentation consistency does not necessarily translate into high reconstruction quality.
Figure \ref{fig:seg_failures} illustrates representative examples. 
Although the predicted segmentation boundaries closely follow the target facial regions, the reconstructed images can exhibit distorted facial geometry, unnatural local structures, or other perceptual artifacts. Interestingly, some of these artifacts appear aligned with semantic boundaries emphasized by the segmentator, suggesting that the reconstruction can satisfy the learned forward model while deviating from the natural image manifold.

These examples highlight a broader challenge when using highly non-injective learned neural networks as forward operators.
Improving the robustness of the interaction between such learned forward operators and generative priors is therefore an important direction for future work.

\begin{figure}[h]
    \centering
    \includegraphics[width=0.4\linewidth]{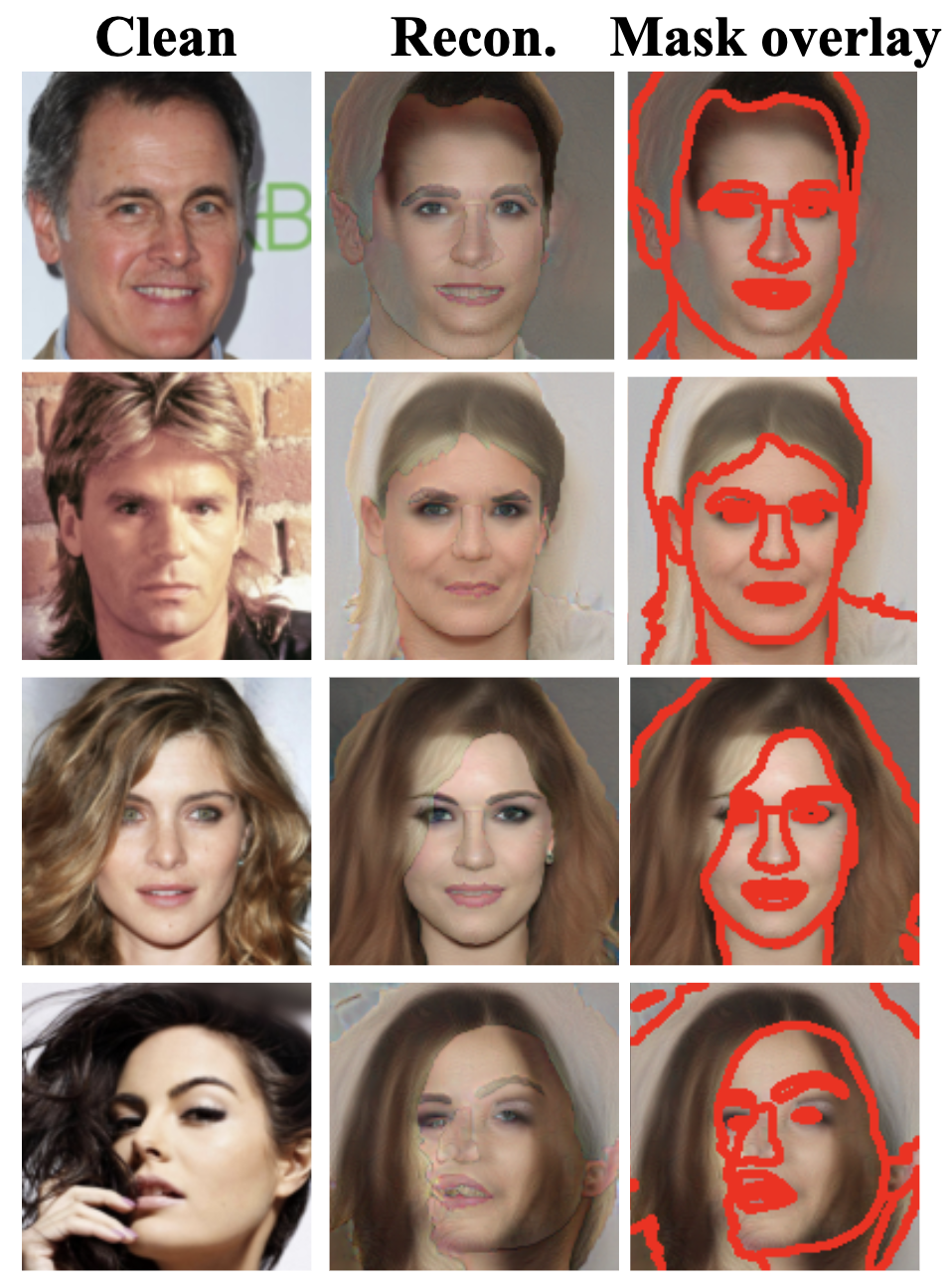}
    \caption{Failure cases for segmentation-guided reconstruction.
    Left: reference images. Middle: PnP-CF reconstructions. Right: reconstructed images with the predicted segmentation boundaries overlaid in red. Despite achieving strong segmentation consistency, some reconstructions exhibit distorted facial geometry and perceptual artifacts. These examples illustrate that satisfying a learned segmentation forward model does not necessarily guarantee a visually plausible reconstruction.}
\label{fig:seg_failures}
\end{figure}

\subsection{Experiments on default hyperparameter settings}
\label{app:default_params}

Table \ref{tab:default_params} report results of PnP-CF run with the default (non-scheduled) hyperparameter settings: a constant penalty $\rho_k = \rho$ (i.e., $\beta=0$) and a linear time schedule $t_k = k/K$ (i.e., $\gamma=1$), across all six inverse problems on both datasets. 
For CelebA, we use $K=200$ iterations and $\rho=1000$; for AFHQ-Cat, we use $K=250$ iterations and $\rho=500$. 
Both datasets use $n=5$ denoiser calls per iteration.  
All other settings are held fixed across the six tasks within a dataset rather than tuned per task.
For the linear inverse problems, we use the more challenging degradations described in section \ref{experiments}, so the results are not directly comparable to those reported by \citet{kumarpnpadmm}.

\begin{table}[ht]
\centering
\scriptsize
\setlength{\tabcolsep}{3pt}
\caption{PnP-CF results on reconstructed vs. noisy/degraded inputs with default (non-scheduled) hyperparameter settings on CelebA and AFHQ-Cat.}
\label{tab:default_params}
\begin{tabular}{ll cc cc cc}
\toprule
& & \multicolumn{2}{c}{PSNR$\uparrow$} & \multicolumn{2}{c}{SSIM$\uparrow$} & \multicolumn{2}{c}{LPIPS$\downarrow$} \\
\cmidrule(lr){3-4} \cmidrule(lr){5-6} \cmidrule(lr){7-8}
Dataset & Task & Rec & Noisy & Rec & Noisy & Rec & Noisy \\
\midrule
\multirow{6}{*}{CelebA}
& Box Inpainting               & 24.71 & 18.91 & 0.76 & 0.66 & 0.08 & 0.31 \\
& Random Inpainting         & 30.15 & 11.46 & 0.90 & 0.18 & 0.04 & 1.09 \\
& Super Resolution          & 20.28 & 10.71 & 0.56 & 0.18 & 0.24 & 0.82 \\
& HDR Reconstruction      & 21.97 & 17.03 & 0.79 & 0.50 & 0.06 & 0.11 \\
& Nonlinear Deblurring      & 25.36 & 24.77 & 0.73 & 0.66 & 0.15 & 0.24 \\
& Phase Retrieval           & 27.35 &  9.12 & 0.68 & 0.21 & 0.20 & 1.24 \\
\midrule
\multirow{6}{*}{AFHQ-Cat}
& Box Inpainting               & 21.48 & 18.31 & 0.69 & 0.64 & 0.20 & 0.35 \\
& Random Inpainting         & 30.61 & 12.41 & 0.85 & 0.20 & 0.07 & 1.13 \\
& Super Resolution          & 18.35 & 11.67 & 0.49 & 0.31 & 0.44 & 0.95 \\
& HDR Reconstruction      & 26.21 & 16.60 & 0.83 & 0.49 & 0.09 & 0.17 \\
& Nonlinear Deblurring      & 22.01 & 20.28 & 0.54 & 0.45 & 0.36 & 0.54 \\
& Phase Retrieval           & 15.06 &  9.70 & 0.37 & 0.31 & 0.52 & 1.36 \\
\bottomrule
\end{tabular}
\end{table}

\subsection{Experiment on baseline settings}
\label{appendix:baselines}

For our three linear inverse problems, the settings used for the results in the main paper (section \ref{experiments}) are substantially more challenging than those considered in \citet{martin2025pnpflow}. We deliberately adopt these harder settings to more clearly distinguish the performance of different methods and demonstrate the robustness of our approach under more severe degradations.

We therefore also evaluate all methods under the milder settings used by \citet{martin2025pnpflow} and report the results in table~\ref{tab:milder_settings}. Specifically, we consider box inpainting with a centered $40\times40$ (CelebA) / $80\times80$ (AFHQ-Cat) mask, random inpainting with $70\%$ of pixels masked, and super-resolution with downsampling factors of $\times2$ (CelebA) / $\times4$ (AFHQ-Cat), while keeping the noise levels unchanged.

The results show that our method remains competitive with the baselines on linear inverse problems under these milder degradation settings.

\begin{table}[ht]
\centering
\caption{Reconstruction quality on CelebA and AFHQ-Cat for linear inverse problems under the milder degradation settings used in \citet{martin2025pnpflow}. Best results are shown in \textbf{bold}.}
\label{tab:milder_settings}
\scriptsize
\setlength{\tabcolsep}{4pt}
\begin{tabular}{l l ccc ccc ccc}
\toprule
& & \multicolumn{3}{c}{Box inpainting} & \multicolumn{3}{c}{Super-resolution} & \multicolumn{3}{c}{Random inpainting} \\
\cmidrule(lr){3-5} \cmidrule(lr){6-8} \cmidrule(lr){9-11}
Dataset & Method & PSNR$\uparrow$ & LPIPS$\downarrow$ & SSIM$\uparrow$ & PSNR$\uparrow$ & LPIPS$\downarrow$ & SSIM$\uparrow$ & PSNR$\uparrow$ & LPIPS$\downarrow$ & SSIM$\uparrow$ \\
\midrule
\multirow{7}{*}{CelebA}
& PnP-CF (ours) & \textbf{31.71} & \textbf{0.02} & \textbf{0.95} & 31.37 & \textbf{0.03} & 0.86 & \textbf{33.99} & \textbf{0.02} & \textbf{0.95} \\
& PnP-Flow      & 30.82 & 0.04 & 0.94 & 31.31 & 0.06 & 0.91 & 33.92 & \textbf{0.02} & \textbf{0.95} \\
& FLOWER        & 31.66 & \textbf{0.02} & \textbf{0.95} & \textbf{32.88} & 0.04 & 0.90 & 33.91 & 0.03 & \textbf{0.95} \\
& FlowDPS       & 28.48 & 0.06 & 0.89 & 25.94 & 0.12 & 0.79 & 26.42 & 0.11 & 0.80 \\
& D-Flow        & 28.28 & 0.06 & 0.85 & 29.40 & 0.04 & 0.83 & 31.96 & 0.03 & 0.91 \\
& Flow Priors   & 27.07 & 0.05 & 0.88 & 28.11 & 0.10 & 0.73 & 29.06 & 0.04 & 0.91 \\
& OT-ODE        & 29.49 & 0.04 & 0.92 & 31.14 & \textbf{0.03} & \textbf{0.93} & 28.68 & 0.05 & 0.87 \\
\midrule
\multirow{7}{*}{AFHQ-Cat}
& PnP-CF (ours) & \textbf{27.86} & \textbf{0.06} & 0.90 & \textbf{28.54} & 0.14 & \textbf{0.78} & 33.39 & \textbf{0.03} & \textbf{0.91} \\
& PnP-Flow      & 27.08 & 0.13 & 0.90 & 28.01 & 0.17 & 0.77 & \textbf{34.11} & 0.05 & \textbf{0.91} \\
& FLOWER        & 27.23 & 0.07 & \textbf{0.92} & 27.06 & 0.27 & 0.77 & 34.10 & 0.05 & 0.90 \\
& FlowDPS       & 24.19 & 0.19 & 0.83 & 26.09 & 0.28 & 0.74 & 27.43 & 0.25 & 0.76 \\
& D-Flow        & 26.35 & 0.14 & 0.77 & 25.14 & 0.16 & 0.62 & 31.38 & 0.09 & 0.85 \\
& Flow Priors   & 24.79 & 0.11 & 0.84 & 25.13 & 0.27 & 0.58 & 32.06 & 0.05 & 0.90 \\
& OT-ODE        & 24.47 & 0.10 & 0.87 & 26.09 & \textbf{0.12} & 0.71 & 30.05 & 0.09 & 0.85 \\

\bottomrule
\end{tabular}
\end{table}

%% file: references.bib
@inproceedings{venkatakrishnan2013plug,
  title={Plug-and-play priors for model based reconstruction},
  author={Venkatakrishnan, Singanallur V and Bouman, Charles A and Wohlberg, Brendt},
  booktitle={2013 IEEE Global Conference on Signal and Information Processing},
  year={2013},
  organization={IEEE}
}

@inproceedings{ryu2019plugandplay,
      title={Plug-and-Play Methods Provably Converge with Properly Trained Denoisers}, 
      author={Ernest K. Ryu and Jialin Liu and Sicheng Wang and Xiaohan Chen and Zhangyang Wang and Wotao Yin},
      booktitle={International Conference on Machine Learning},
      eprint={arXiv:1905.05406},
      year={2019},
}

@article{Zhang_2017,
   title={Beyond a Gaussian Denoiser: Residual Learning of Deep CNN for Image Denoising},
   journal={IEEE Transactions on Image Processing},
   publisher={Institute of Electrical and Electronics Engineers},
   author={Zhang, Kai and Zuo, Wangmeng and Chen, Yunjin and Meng, Deyu and Zhang, Lei},
   year={2017},
}

@inproceedings{bora2017compressed,
      title={Compressed Sensing using Generative Models}, 
      author={Ashish Bora and Ajil Jalal and Eric Price and Alexandros G. Dimakis},
      booktitle={International Conference on Machine Learning},
      year={2017},
      eprint={arXiv:1703.03208}
}

@inproceedings{zhu2023denoising,
      title={Denoising Diffusion Models for Plug-and-Play Image Restoration}, 
      author={Yuanzhi Zhu and Kai Zhang and Jingyun Liang and Jiezhang Cao and Bihan Wen and Radu Timofte and Luc Van Gool},
      booktitle={IEEE/CVF Conference on Computer Vision and Pattern Recognition},
      year={2023},
      eprint={arXiv:2305.08995}
}

@inproceedings{mardani2023variational,
      title={A Variational Perspective on Solving Inverse Problems with Diffusion Models}, 
      author={Morteza Mardani and Jiaming Song and Jan Kautz and Arash Vahdat},
      booktitle={International Conference on Learning Representations},
      year={2024},
      eprint={arXiv:2305.04391}
}

@article{li2025decoupled,
      title={Decoupled Data Consistency with Diffusion Purification for Image Restoration}, 
      author={Xiang Li and Soo Min Kwon and Shijun Liang and Ismail R. Alkhouri and Saiprasad Ravishankar and Qing Qu},
      journal={arXiv:2403.06054},
      year={2025},
}

@inproceedings{martin2025pnpflow,
      title={PnP-Flow: Plug-and-Play Image Restoration with Flow Matching}, 
      author={Ségolène Martin and Anne Gagneux and Paul Hagemann and Gabriele Steidl},
      booktitle={International Conference on Learning Representations},
      year={2025},
      eprint={arXiv:2410.02423}
}

@inproceedings{pourya2026flower,
      title={{FLOWER}: A Flow-Matching Solver for Inverse Problems},
      author={Mehrsa Pourya and Bassam El Rawas and Michael Unser},
      booktitle={International Conference on Learning Representations},
      year={2026}
}

@inproceedings{chung2024diffusion,
      title={Diffusion Posterior Sampling for General Noisy Inverse Problems}, 
      author={Hyungjin Chung and Jeongsol Kim and Michael T. Mccann and Marc L. Klasky and Jong Chul Ye},
      year={2023},
      booktitle={International Conference on Learning Representations},
      eprint={arXiv:2209.14687}
}

@inproceedings{lipman2023flowmatching,
      title={Flow Matching for Generative Modeling}, 
      author={Yaron Lipman and Ricky T. Q. Chen and Heli Ben-Hamu and Maximilian Nickel and Matt Le},
      booktitle={International Conference on Learning Representations},
      year={2023},
      eprint={arXiv:2210.02747}
}

@article{boyd_distributed_2010,
	title = {Distributed {Optimization} and {Statistical} {Learning} via the {Alternating} {Direction} {Method} of {Multipliers}},
	journal = {Foundations and Trends in Machine Learning},
	author = {Boyd, Stephen and Parikh, Neal and Chu, Eric and Peleato, Borja and Eckstein, Jonathan},
	year = {2010}
}

@inproceedings{yang2015facialparts,
      title={From Facial Parts Responses to Face Detection: A Deep Learning Approach}, 
      author={Shuo Yang and Ping Luo and Chen Change Loy and Xiaoou Tang},
      booktitle={IEEE International Conference on Computer Vision},
      year={2015},
      eprint={arXiv:1509.06451}
}

@inproceedings{choi2020starganv,
      title={StarGAN v2: Diverse Image Synthesis for Multiple Domains}, 
      author={Yunjey Choi and Youngjung Uh and Jaejun Yoo and Jung-Woo Ha},
      year={2020},
      booktitle={IEEE/CVF Conference on Computer Vision and Pattern Recognition},
      eprint={arXiv:1912.01865}
}

@inproceedings{pokle2024training,
      title={Training-free Linear Image Inverses via Flows}, 
      author={Ashwini Pokle and Matthew J. Muckley and Ricky T. Q. Chen and Brian Karrer},
      booktitle={International Conference on Learning Representations},
      year={2024},
      eprint={arXiv:2310.04432}
}

@inproceedings{benhamu2024dflow,
      title={D-Flow: Differentiating through Flows for Controlled Generation}, 
      author={Heli Ben-Hamu and Omri Puny and Itai Gat and Brian Karrer and Uriel Singer and Yaron Lipman},
      booktitle={International Conference on Learning Representations},
      year={2024},
      eprint={arXiv:2402.14017}
}

@inproceedings{zhang2025flowpriors,
      title={Flow Priors for Linear Inverse Problems via Iterative Corrupted Trajectory Matching}, 
      author={Yasi Zhang and Peiyu Yu and Yaxuan Zhu and Yingshan Chang and Feng Gao and Ying Nian Wu and Oscar Leong},
      booktitle={Advances in Neural Information Processing Systems},
      year={2024},
      eprint={arXiv:2405.18816}
}

@article{hestenes1952methods,
  author = {Hestenes, Magnus R and Stiefel, Eduard},
  journal = {Journal of research of the National Bureau of Standards},
  title = {Methods of conjugate gradients for solving linear systems},
  year = 1952
}

@inproceedings{shrestha2026,
      title={Taming Score-Based Denoisers in ADMM: A Convergent Plug-and-Play Framework}, 
      author={Rajesh Shrestha and Xiao Fu},
      booktitle={International Conference on Learning Representations},
      eprint={arXiv:2603.10281},
      year={2026},
}

@inproceedings{liu2022flowstraight,
      title={Flow Straight and Fast: Learning to Generate and Transfer Data with Rectified Flow}, 
      author={Xingchao Liu and Chengyue Gong and Qiang Liu},
      booktitle={International Conference on Learning Representations},
      eprint={arXiv:2209.03003},
      year={2023},
}

@book{boyd2004convex,
  title={Convex Optimization},
  author={Boyd, Stephen and Vandenberghe, Lieven},
  publisher={Cambridge University Press},
  year={2004},
}

@inproceedings{tran2021blurencoding,
  author    = {Tran, Phong and Tran, Anh Tuan and Phung, Quynh and Hoai, Minh},
  title     = {Explore Image Deblurring via Encoded Blur Kernel Space},
  booktitle = {Proceedings of the IEEE/CVF Conference on Computer Vision and Pattern Recognition (CVPR)},
  year      = {2021},
}

@inproceedings{karras2018progressive,
  title={Progressive Growing of {GAN}s for Improved Quality, Stability, and Variation},
  author={Karras, Tero and Aila, Timo and Laine, Samuli and Lehtinen, Jaakko},
  booktitle={International Conference on Learning Representations},
  year={2018}
}

@inproceedings{lee2020maskgan,
  title={MaskGAN: Towards Diverse and Interactive Facial Image Manipulation},
  author={Lee, Cheng-Han and Liu, Ziwei and Wu, Lingyun and Luo, Ping},
  booktitle={IEEE Conference on Computer Vision and Pattern Recognition (CVPR)},
  year={2020}
}

@inproceedings{kim2025flowdps,
  title={{FlowDPS}: Flow-Driven Posterior Sampling for Inverse Problems},
  author={Kim, Jeongsol and Kim, Bryan Sangwoo and Ye, Jong Chul},
  booktitle={Proceedings of the IEEE/CVF International Conference on Computer Vision (ICCV)},
  pages={12328--12337},
  year={2025}
}

@article{ParikhBoyd2014,
  title={Proximal algorithms},
  author={Parikh, Neal and Boyd, Stephen},
  journal={Foundations and Trends in optimization},
  volume={1},
  number={3},
  pages={127--239},
  year={2014},
  publisher={Emerald Publishing Limited}
}

@article{saharia2022image,
  title={Image super-resolution via iterative refinement},
  author={Saharia, Chitwan and Ho, Jonathan and Chan, William and Salimans, Tim and Fleet, David J and Norouzi, Mohammad},
  journal={IEEE transactions on pattern analysis and machine intelligence},
  volume={45},
  number={4},
  pages={4713--4726},
  year={2022},
  publisher={IEEE}
}

@article{batzolis2021conditional,
  title={Conditional image generation with score-based diffusion models},
  author={Batzolis, Georgios and Stanczuk, Jan and Sch{\"o}nlieb, Carola-Bibiane and Etmann, Christian},
  journal={arXiv preprint arXiv:2111.13606},
  year={2021}
}

@inproceedings{ronneberger2015u,
  title={U-net: Convolutional networks for biomedical image segmentation},
  author={Ronneberger, Olaf and Fischer, Philipp and Brox, Thomas},
  booktitle={International Conference on Medical image computing and computer-assisted intervention},
  year={2015},
  organization={Springer}
}

@article{armijo1966minimization,
  title={Minimization of functions having Lipschitz continuous first partial derivatives},
  author={Armijo, Larry},
  journal={Pacific Journal of mathematics},
  volume={16},
  number={1},
  pages={1--3},
  year={1966},
  publisher={Mathematical Sciences Publishers}
}

@inproceedings{kumarpnpadmm,
  title={Plug-and-Play ADMM for Inverse Problems with Flow Matching Denoiser},
  author={Kumar, Dibyanshu and Gregorová, Magda},
  booktitle={Advances in Neural Information Processing Systems},
  year={2026}
}
